\documentclass[twoside,11pt]{article}

\usepackage{blindtext}

\usepackage{jmlr2e}

\usepackage[utf8]{inputenc} %
\usepackage[T1]{fontenc}    %
\usepackage[
    dvipsnames
]{xcolor}                   %
\usepackage{url}            %
\usepackage{booktabs}       %
\usepackage{amsfonts}       %
\usepackage{nicefrac}       %
\usepackage{microtype}      %
\usepackage{amsmath}        %
\usepackage{amssymb}        %
\usepackage{amsfonts}       %
\usepackage{mathtools}      %
\usepackage{bm, bbm}        %
\usepackage[
    scr=boondoxo,
    scrscaled=1.05
]{mathalpha}                %
\usepackage{csquotes}       %
\usepackage{comment}        %
\usepackage{dsfont}         %
\hypersetup{ hidelinks }
\usepackage{enumitem}       %
\usepackage{appendix}       %
\usepackage{apptools}       %
\usepackage{blindtext}      %
\usepackage{tikz}           %
\usepackage{graphicx}                 %
\usepackage{tikzscale}      %
\usepackage[
    nameinlink
]{cleveref}                 %
\usepackage{caption}        %
\usepackage{subcaption}     %

\definecolor{orange}{rgb}{1,0.4,0.0}

\DeclarePairedDelimiterXPP{\KL}[2]{D_\textnormal{KL}}{(}{)}{}{%
#1\:\delimsize\|\:#2%
}
\DeclarePairedDelimiterXPP{\RD}[2]{D_{$\alpha$}}{(}{)}{}{%
#1\:\delimsize\|\:#2%
}

\DeclarePairedDelimiterXPP\Prob[1]{\mathbb{P}}{\lbrace}{\rbrace}{}{

#1}

\DeclarePairedDelimiterXPP{\lnorm}[2]{}{\lVert}{\rVert}{_{#2}}{#1}

\newcommand{\bA}{\ensuremath{\mathbb{A}}}

\newcommand{\bE}{\ensuremath{\mathbb{E}}}

\newcommand{\bI}{\ensuremath{\mathds{1}}}

\newcommand{\bN}{\ensuremath{\mathbb{N}}}

\newcommand{\bP}{\ensuremath{\mathbb{P}}}
\newcommand{\bQ}{\ensuremath{\mathbb{Q}}}
\newcommand{\bR}{\ensuremath{\mathbb{R}}}

\newcommand{\bV}{\ensuremath{\mathbb{V}}}
\newcommand{\bW}{\ensuremath{\mathbb{W}}}

\newcommand{\cA}{\ensuremath{\mathcal{A}}}
\newcommand{\cB}{\ensuremath{\mathcal{B}}}

\newcommand{\cE}{\ensuremath{\mathcal{E}}}
\newcommand{\cF}{\ensuremath{\mathcal{F}}}
\newcommand{\cG}{\ensuremath{\mathcal{G}}}

\newcommand{\cK}{\ensuremath{\mathcal{K}}}

\newcommand{\cN}{\ensuremath{\mathcal{N}}}
\newcommand{\cO}{\ensuremath{\mathcal{O}}}
\newcommand{\cP}{\ensuremath{\mathcal{P}}}

\newcommand{\cW}{\ensuremath{\mathcal{W}}}
\newcommand{\cX}{\ensuremath{\mathcal{X}}}

\newcommand{\cZ}{\ensuremath{\mathcal{Z}}}

\newcommand{\ccF}{\ensuremath{\mathscr{F}}}

\newcommand{\ccR}{\ensuremath{\mathscr{R}}}

\newcommand{\relent}{\textup{D}}
\newcommand{\relentber}{\textup{d}}

\newcommand{\poprisk}[1]{\ccR(#1)}
\newcommand{\emprisk}[2]{\widehat{\ccR}(#1,#2)}

\DeclareMathOperator*{\esssup}{ess\,sup}

\usepackage{lastpage}
\jmlrheading{25}{2024}{1-\pageref{LastPage}}{10/23; Revised
3/24}{3/24}{23-1360}{Borja Rodríguez-Gálvez, Ragnar Thobaben, and Mikael Skoglund}

\ShortHeadings{More PAC-Bayes bounds}{Rodríguez-Gálvez, Thobaben, and Skoglund}
\firstpageno{1}

\begin{document}

\title{More PAC-Bayes bounds: From bounded losses, to losses with general tail behaviors, to anytime validity}

\author{\name Borja Rodríguez-Gálvez \email borjarg@kth.se \\
       \addr Division of Information Science and Engineering (ISE) \\
       KTH Royal Institute of Technology\\
       Stockholm, Sweden
       \AND
       \name Ragnar Thobaben\email ragnart@kth.se \\
       \addr Division of Information Science and Engineering (ISE) \\
       KTH Royal Institute of Technology\\
       Stockholm, Sweden
       \AND
       \name Mikael Skoglund\email skoglund@kth.se \\
       \addr Division of Information Science and Engineering (ISE) \\
       KTH Royal Institute of Technology\\
       Stockholm, Sweden}

\editor{Ohad Shamir}

\maketitle

\begin{abstract}%
In this paper, we present new high-probability PAC-Bayes bounds for different types of losses.  Firstly, for losses with a bounded range, we recover a strengthened version of Catoni's bound that holds uniformly for all parameter values. This leads to new fast-rate and mixed-rate bounds that are interpretable and tighter than previous bounds in the literature. In particular, the fast-rate bound is equivalent to the Seeger--Langford bound. Secondly, for losses with more general tail behaviors, we introduce two new parameter-free bounds: a PAC-Bayes Chernoff analogue when the loss' cumulative generating function is bounded, and a bound when the loss' second moment is bounded. These two bounds are obtained using a new technique based on a discretization of the space of possible events for the ``in probability'' parameter optimization problem. This technique is both simpler and more general than previous approaches optimizing over a grid on the parameters' space. Finally, using a simple technique that is applicable to any existing bound,
we extend all previous results to anytime-valid bounds.
\end{abstract}


\section{Introduction}
\label{sec:introduction}

A learning algorithm $\bA$ is a (possibly randomized) mechanism that generates a hypothesis $w \in \cW$ of the solution of a certain problem %
given a sequence of $n$ training data samples $s \coloneqq (z_1, \ldots, z_n)$, or \emph{training set}. The performance of a hypothesis $w$ on %
an instance $z$ of the problem %
is described by a loss function $\ell(w,z)$. Hence, if %
the problem's instances follow a distribution $\bP_Z$, the goal is to produce a  hypothesis $w$ that attains a low \emph{population risk} $\poprisk{w} \coloneqq \bE \ell(w,Z)$, which is defined as the expected loss of the hypothesis $w$ on samples $Z$ drawn randomly from the problem's distribution $\bP_Z$. 

\looseness=-1 Often, computing the population risk is not feasible. This is because, in general, the distribution $\bP_Z$ is unknown or intractable. However, having access to a training set $s$, a computable proxy for the population risk is the \emph{empirical risk} $\emprisk{w}{s} \coloneqq \frac{1}{n} \sum_{i=1}^n \ell(w, z_i)$, which is defined as the average loss of the hypothesis $w$ on the samples from the training set $s$.

There are different attempts at characterizing the population risk based on the decomposition 
\begin{equation*}
    \label{eq:population_risk_decomposition}
    \poprisk{w} = \emprisk{w}{s} + \underbrace{\Big( \poprisk{w} - \emprisk{w}{s} \Big)}_{\mathclap{ \textnormal{generalization error}}}.
\end{equation*}

\emph{Probably approximately correct (PAC)} theory gives bounds on the generalization error that hold with a probability larger than a certain threshold. Classically, these bounds depend only on the complexity of the hypothesis space $\cW$, which is measured by, for example, the Vapnik--Cherovenkis (VC) dimension or the Rademacher complexity. See, for example, the book from \citet{shalev2014understanding} for a pedagogical exposition of the topic.

In this work, we are concerned with \emph{PAC-Bayesian bounds} \citep{shawe1996framework,mcallester1998some,mcallester1999pac, mcallester2003pac}, which also consider the dependence of the hypothesis returned by the algorithm $W = \bA(S)$ on the training set S. These bounds are often of the following type: ``for every $\beta \in (0,1)$, with probability no smaller than $1 - \beta$ 
\begin{equation*}
    \bE^S \poprisk{W} \leq \bE^S \emprisk{W}{S} + \alpha_{\textnormal{PAC-Bayes}}(S)\textnormal,\text{''}
\end{equation*}
where the probability is taken with respect to the sampling of the training set $S \sim \bP_S$ and $\bE^S$ denotes the conditional expectation operator with respect to the $\sigma$-algebra induced by $S$. The term $\alpha_{\textnormal{PAC-Bayes}}(S)$ describes the discrepancy between the population and empirical risks and is a random variable depending on $S$. Intuitively, this term (i) decreases with $n$ as a better characterization of the risk is possible the more samples are available; (ii) increases with $\nicefrac{1}{\beta}$ as certainty comes with a price; and (iii) decreases as the hypothesis becomes less statistically dependent on the training set, as intuitively motivated by the fact that the empirical risk is always an unbiased estimate of the population risk in the extreme that the algorithm produces a fully independent hypothesis, that is $\bE \poprisk{W} = \bE \emprisk{W}{S}$.
In this paper, similarly to~\citet{hellstrom2020generalization}, we shall agree to the convention that the bounds are of \emph{high probability} if the dependence on $\nicefrac{1}{\beta}$ is logarithmic, that is, $\log \nicefrac{1}{\beta}$. The review from~\citet{alquier2021user} offers an extensive introduction to PAC-Bayes theory.

\citet{mcallester1998some, mcallester1999pac, mcallester2003pac} showed the original PAC-Bayes bound considering bounded losses. 
The bound\footnote{The bound written is the one obtained relaxing the Seeger--Langford bound~\citep{langford2001bounds, seeger2002pac} via a lower bound on the binary relative entropy using Pinsker's inequality. The term dependence $\xi(n)$ with the number os samples $n$ is the one established by~\citet{maurer2004note}. See, for example, Section 2.2 of~\citet{tolstikhin2013pac}.} states that for any prior $\bQ_W$ independent of $S$ and every $\beta \in (0,1)$, with probability no less than $1-\beta$
\begin{equation}
\label{eq:mcallester_with_germain_and_mauer_pac_bayes}
    \bE^S \poprisk{W} \leq \bE^S \emprisk{W}{S} + \sqrt{\frac{\relent(\bP_W^S \Vert \bQ_W) + \log \frac{\xi(n)}{\beta}}{2n}}
\end{equation}
\looseness=-1 simultaneously for every Markov kernel $\bP_W^S$, where $\xi(n) \in \big[\sqrt{n}, 2 + \sqrt{2n} \ \big]$~\citep{maurer2004note} and the dependency between the hypothesis and the dataset is measured by the relative entropy $\relent( \bP_W^S \Vert \bQ_W)$ of the algorithm's hypothesis kernel $\bP_W^S$, or \emph{posterior}, with respect to an arbitrary data-independent distribution on the hypothesis space $\bQ_W$, or \emph{prior}.\footnote{The range of $\xi(n)$ is usually set to $[\sqrt{n}, 2\sqrt{n}]$ for all $n \geq 1$ as per the analysis of \citet{maurer2004note} and empirical further analysis of \citet[Lemma 19]{germain2015risk}. From \citet[Theorem 1]{maurer2004note}, we can observe that the tighter $2 + \sqrt{2n}$ is valid as an upper bound for all $n \geq 2$ and the case where $n=1$ can be verified empirically using the bound from \citet[Lemma 19]{germain2015risk}.} The dependency term $\relent(\bP_W^S \Vert \bQ)$ inside the square root plays the role of the complexity term of the classical PAC bounds, while the extra dependence $\xi(n)$ on the number of samples comes from the concentration of the empirical risk around the population risk (see~\citet{maurer2004note} for the details). Finally, the term $\log \nicefrac{1}{\beta}$ is the confidence penalty of being a high-probability bound. Sometimes, to simplify the discussion we will refer to this structure as the (normalized) \emph{dependence-confidence} term and we define it as $\mathfrak{C}_{n, \beta,S} \coloneqq \frac{1}{n} \big(\relent(\bP_W^S \Vert \bQ_W) + \log \nicefrac{1}{\beta}\big)$, which in the case of~\eqref{eq:mcallester_with_germain_and_mauer_pac_bayes} corresponds to $\mathfrak{C}_{2n, \nicefrac{\beta}{\xi(n)}, S}$.

Many works on PAC-Bayes bounds have focused on two main tasks: (i) refining the bound to better characterize the population risk for bounded losses and (ii) extending this bound relaxing their assumptions or their setting.

In the first front, \citet{langford2001bounds, seeger2002pac} and \citet{catoni2003pac, catoni2007pac} developed more accurate bounds for estimating the population risk for bounded losses. 
However, either these bounds are not easily interpretable, minimizing them to find an appropriate posterior is hard, or they depend on an arbitrary parameter that needs to be selected \emph{before} the draw of the data. To address these issues \citet{tolstikhin2013pac}, \citet{thiemann2017strongly}, and \citet{rivasplata2019pac} relaxed the Seeger--Langford bound~\citep{langford2001bounds, seeger2002pac} to find more interpretable bounds where an approximate minimization to find a suitable posterior is possible. To contribute in this front:
\begin{quote}
    \textbf{Contribution 1.}  In \Cref{sec:specialized_pac_bayes_bounds_bounded_losses}, we show an %
    alternative proof of a strengthened version of \citet{catoni2007pac}'s PAC-Bayes bound that holds uniformly for all values of the parameter $\lambda$ (\Cref{th:catoni_pac_bayes_uniform}). We then build on this bound to show tighter fast-rate (\Cref{th:fast_rate_bound_strong}) and mixed-rate (\Cref{th:mixed_rate_bound}) bounds that are interpretable and help us to clarify the relationship between the population risk, the empirical risk, and the relative entropy of the algorithm's posterior with respect to the prior. A mixed-rate bound is a bound with a mixture of a fast rate, and an amortized slow rate. The precise meaning becomes clear looking at \Cref{th:mixed_rate_bound}. The fast-rate bound of \Cref{th:fast_rate_bound_strong} is of particular interest since it is equivalent to the Seeger--Langford bound~\citep{langford2001bounds, seeger2002pac}. This 
    reveals two significant insights: (i) that a linear combination of the empirical risk and the dependence-confidence term characterizes the bound and (ii) that the optimal posterior is a Gibbs distribution with a data-dependent ``temperature''. 
\end{quote}

\citet{wu2022split} derived a ``split-kl'' inequality that competes with the Seeger--Langford bound~\citep{langford2001bounds, seeger2002pac} for ternary losses and  \citet{jang2023tighter} proved an even tighter bound via ``coin-betting''. However, their bounds still neither are easily interpretable nor directly aid to the selection of an appropriate posterior. Moreover, there are other advances in this front when further quantities are considered.  If the variance is known, \citet[Theorem 8]{seldin2012pac} and \citet[Theorem 9]{wu2021chebyshev} introduced, respectively, PAC-Bayes analogues to Bernstein and Bennet inequalities. The PAC-Bayes Bernstein inequality was later improved by further bounding the variance using an empirical estimate of that quantity~\citep[Theorems 3 and 4]{tolstikhin2013pac}. Finally, \citet{mhammedi2019pac} derived a PAC-Bayes analogue to the ``un-expected Bernstein inequality'' where they use an empirical estimate of the second moment.

In the second front, \citet{guedj2021still} and \citet{hellstrom2020generalization} extended \citet{mcallester2003pac}'s bound to subgaussian losses, resulting in the same rate as the original bound~\eqref{eq:mcallester_with_germain_and_mauer_pac_bayes}. However, the proof of these new bounds contains a small mistake. They derive intermediate PAC-Bayes bounds depending on a parameter $\lambda$ that needs to be selected \emph{before} the draw of the training data, and then they optimize this parameter without paying the necessary union bound price~\citep[Remark 14]{banerjee2021information}.\footnote{\citet{hellstrom2021corrections} later corrected this issue using unique subgaussian properties~\citep[Theorem 2.6]{wainwright2019high}.} This leads to vacuous bounds as potentially an infinite number of parameters can be optimal for different data and it is a known standing problem in the PAC-Bayes literature~\citep[Section 2.1.4]{alquier2021user}. To address this issue:
\begin{quote}
    \textbf{Contribution 2.} In \Cref{sec:pac_bayes_beyond_bounded_losses}, we devise a proof technique that allows us to bypass this optimization subtlety. We use this technique to extend \citet{mcallester2003pac}'s bound to losses with more general tail behaviors. First, we derive a PAC-Bayes Chernoff analogue (\Cref{th:pac_bayes_chernoff_analogue}) that specializes to the bounds of \citet{guedj2021still} and \citet{hellstrom2020generalization, hellstrom2021corrections} for subgaussian losses. After that, we derive a parameter-free PAC-Bayes bound requiring only that the loss has a bounded second moment (\Cref{th:parameter_free_anytime_valid_bounded_2nd_moment}). This last bound is of the nature of \citep{kuzborskij2019efron} and is obtained by optimizing the parameter of \citet{wang2015pac}'s bound on martingales in \Cref{subapp:closed_form_parameter_free_wang}.
    The proposed technique is simpler and more general than previous approaches that generate a grid on the parameters' space, optimize the parameter over that grid, and pay the union bound price. Contrary to our technique, these approaches either can't generate a parameter-free bound~\citep{langford2001not, catoni2003pac} or need to craft the grid in a case-to-case basis and need that an explicit solution of the optimal parameter exists~\citep{seldin2012pac}, which may not be the case  (see \Cref{subsubsec:related_work_optimization}). Therefore, the proposed technique is of independent interest for the development of future bounds ``in probability''.
\end{quote}
Other works also developed PAC-Bayes bounds with more general tail behaviors~\citep{catoni2004statistical, alquier2006transductive,  kuzborskij2019efron, haddouche2023pacbayes, alquier2018simpler, holland2019pac, haddouche2021pac,  chugg2023unified}. However, most of these bounds either are not of high probability, or contain terms that often make the bounds non-decreasing with the number of samples $n$, or decrease at a slower rate than \eqref{eq:mcallester_with_germain_and_mauer_pac_bayes} when restricted to the bounded case, or also depend on parameters that need to be chosen \emph{before} the draw of the training data.

Recently, some research has focused on developing PAC-Bayes bounds that hold \emph{simultaneously} for all numbers of samples $n$~\citep{chugg2023unified, jang2023tighter, haddouche2023pacbayes}. These bounds are particularly useful for online learning algorithms that process data sequentially. These \emph{anytime-valid} (or \emph{time-uniform}) PAC-Bayes bounds are typically based on supermartingales and \citet{ville1939etude}'s extension of Markov's inequality. To contribute in this end:
\begin{quote}
    \textbf{Contribution 3.} In \Cref{sec:anytime_valid_pac_bayes_bounds}, we note that every PAC-Bayes bound can be extended to an anytime-valid one at a union bound cost (\Cref{th:standard_to_anytime_valid}). For high-probability PAC-Bayes bounds, this cost is small.
\end{quote}

    Finally, note that while the relative entropy is widely used as the dependency measure in PAC-Bayes bounds due to its simplicity, interpretability from an information-theoretic perspective, and mathematical tractability, it has been shown that there are situations where an algorithm generalizes but this measure is large, making the bounds vacuous~\citep{bassily2018learners,livni2020limitation,haghifam2022limitations,nagarajan2019uniform,nachum2023fantastic}. The study of different dependency measures for PAC-Bayes bounds is outside the scope of this paper. Nonetheless, we refer the reader to other literature substituting the relative entropy as the dependency measure by different metrics like other $f$-divergences~\citep{esposito2021generalization,ohnishi2021novel,kuzborskij2024better}, Rényi divergences~\citep{begin2016pac,esposito2021generalization,hellstrom2020generalization}, or integral probability metrics like the Wasserstein distance~\citep{amit2022integral,haddouche2023wasserstein,viallard2024learning,viallard2024tighter}.

\section{Specialized PAC-Bayes bounds for bounded losses}
\label{sec:specialized_pac_bayes_bounds_bounded_losses}

This section is separated into two parts. In \Cref{subsec:review_pac_bayes_bounded_losses}, we review the state of the art of PAC-Bayes bounds for bounded losses. Then, in \Cref{subsec:seeger_langford_to_catoni_and_new_fast_and_mixed_rate_bounds} we %
give an alternative proof of a strengthened version of \citet{catoni2007pac}'s parameterized bound that holds \emph{simultaneously} for all values of the parameter. After that, we show that relaxing this strengthened bound (\Cref{th:catoni_pac_bayes_uniform}) yields fast-rate (\Cref{th:fast_rate_bound_strong} and \Cref{cor:fast_rate_bound}) and mixed-rate (\Cref{th:mixed_rate_bound}) bounds tighter than \citet{thiemann2017strongly}'s fast-rate and \citet{tolstikhin2013pac}'s and \citet{rivasplata2019pac}'s mixed-rate bounds.

\subsection{A review of PAC-Bayes bounds for bounded losses}
\label{subsec:review_pac_bayes_bounded_losses}

There are many important inequalities in the PAC-Bayes literature, especially for the case where the loss is bounded. These bounds are often presented for losses with a range in $[0,1]$, which includes the interesting 0--1 loss for classification tasks. The Seeger--Langford~\citep{langford2001bounds, seeger2002pac} and \citet[Theorem 1.2.6]{catoni2007pac}'s bounds are known to be (two of) the tightest bounds in this setting (cf. \citep{foong2021tight}). Both of them can be derived from \citet[Theorem 2.1]{germain2009pac}'s convex function bound. Below we state the extension from \citet{begin2014pac} that lifts the double absolute continuity requirement from the original statement noted by \citet{haddouche2021pac}.

\begin{theorem}[{\bf \citet[Theorem 4]{begin2016pac}}]
\label{th:germain_convex_pac_bayes}
    \sloppy Consider a loss function $\ell$ with bounded range $[0,1]$, let $\bQ_W$ be any prior independent of $S$, and let $W'$ be distributed according to $\bQ_W$. Then, for every convex function $f: [0,1] \times [0,1] \to \bR$ such that $\bE \big[ \exp \big( n f \big( \emprisk{W'}{s}, \poprisk{W'} \big) \big) \big] < \infty$ for all $s \in \cZ^n$, and every $\beta \in (0,1)$, with probability no smaller than $1-\beta$
    \begin{equation*}
        \label{eq:germain_convex_pac_bayes}
        f \big( \bE^S \emprisk{W}{S} , \bE^S \poprisk{W} \big) \leq \frac{1}{n} \bigg[ \relent \big( \bP_W^S \Vert \bQ_W \big) + \log \frac{1}{\beta} + \log \bE \Big[e^{n f \big(  \emprisk{W'}{S} , \poprisk{W'} \big)} \Big]\bigg]
    \end{equation*}
    holds \emph{simultaneously} for every posterior $\bP_W^S$.
\end{theorem}

This general bound is useful because an appropriate choice of the convex function $f$ can be used to recover \citet{mcallester2003pac}'s bound.\footnote{It is often mentioned that this is done choosing $f(p,q) = 2(p-q)^2$. However, technically, to use \citet{mcallester2003pac}'s proof $f(p,q) = \nicefrac{(2n-1)}{n}(p-q)^2$ should be used instead.}
Similarly, choosing $f(p,q)= \relentber(p \Vert q) \coloneqq \relent(\textnormal{Ber}(p) \Vert \textnormal{Ber}(q))$ combined with \citet{maurer2004note}'s trick recovers the improved Seeger--Langford bound~\citep{langford2001bounds, seeger2002pac}, and choosing $f(p,q) = - \log \big(1 - q(1-e^{\nicefrac{-\lambda}{n}})\big) - \nicefrac{\lambda p}{n}$ recovers \citet[Theorem 1.2.6]{catoni2007pac}'s bound. 

\begin{theorem}
\label{th:seeger_langford_pac_bayes}
{\bf (Improved Seeger--Langford bound \citep{langford2001bounds,seeger2002pac,maurer2004note})}
    Consider a loss function $\ell$ with bounded range $[0,1]$ and let $\bQ_W$ be any prior independent of $S$. Then, for every $\beta \in (0,1)$, with probability no smaller than $1-\beta$
    \begin{equation}
        \label{eq:seeger_langford_pac_bayes}
        \relentber \big( \bE^S \emprisk{W}{S} \Vert \bE^S \poprisk{W} \big) \leq \frac{\relent(\bP_W^S \Vert \bQ_W) + \log \frac{\xi(n)}{\beta}}{n}
    \end{equation}
    holds \emph{simultaneously} for every posterior $\bP_W^S$.
\end{theorem}

\begin{theorem}[{\bf \citet[Theorem 1.2.6]{catoni2007pac}}]
\label{th:catoni_pac_bayes}
    Consider a loss function $\ell$ with bounded range $[0,1]$ and let $\bQ_W$ be any prior independent of $S$. Then, for every $\lambda > 0$ and every $\beta \in (0,1)$, with probability no smaller than $1-\beta$
    \begin{equation*}
        \label{eq:catoni_pac_bayes}
        \bE^S \poprisk{W} \leq \frac{1}{1 - e^{- \frac{\lambda}{n}}} \Bigg[1 - e^{- \frac{\lambda \bE^S \emprisk{W}{S}}{n}   - \frac{\relent(\bP_W^S \Vert \bQ_W) + \log \frac{1}{\beta}}{n}} \Bigg]
    \end{equation*}
    holds \emph{simultaneously} for every posterior $\bP_W^S$.
\end{theorem}

The Seeger-Langford bound \citep{langford2001bounds,seeger2002pac} is hindered by its lack of interpretability and the difficulty in minimizing it to find an appropriate posterior $\bP_{W}^{S}$. This is due to the non-convexity of the bound with respect to the posterior $\bP_W^S$~\citep{thiemann2017strongly} as well as the fact that it cannot be expressed explicitly as a function of the empirical risk $\bE^S \emprisk{W}{S}$ and the
dependency term $\relent(\bP_W^S \Vert \bQ_W)$~\citep{germain2009pac}. On the other hand, while \citet{catoni2007pac}'s bound is minimized by the Gibbs posterior proportional to $\bQ_W(w) e^{- \lambda \emprisk{w}{S}}$, it still lacks interpretability and depends on an arbitrary parameter $\lambda$ that has to be chosen \emph{before} the draw of the data.

To remedy these issues, several works relax the Seeger--Langford bound~\citep{langford2001bounds,seeger2002pac} using lower bounds on the relative entropy~\citep{tolstikhin2013pac,thiemann2017strongly,rivasplata2019pac}. Since \citet{mcallester2003pac}'s bound~\eqref{eq:mcallester_with_germain_and_mauer_pac_bayes} is recovered with the standard Pinsker's inequality~\citep[Theorem 7.9]{polyanskiy2022lecture}, these works employ different relaxations of the stronger \citet{marton1996measure}'s bound, cf. \citep[Corollaries 2.19 and 2.20]{seldinNotes}. \citet{tolstikhin2013pac} use \citep[Corollary 2.20]{seldinNotes} and \citet{thiemann2017strongly} and \citet{rivasplata2019pac} use \citep[Corollary 2.19]{seldinNotes}. The latter bound results in an intractable PAC-Bayes bound; therefore \citet{thiemann2017strongly} relax it using the inequality $\sqrt{xy} \leq \frac{1}{2}(\lambda x + \nicefrac{y}{\lambda})$ for all $\lambda > 0$ to obtain a \emph{fast-rate} bound, and \citet{rivasplata2019pac} solve the resulting quadratic inequality for $\sqrt{\bE^S \poprisk{W}}$ to obtain a \emph{mixed-rate} bound.

\begin{theorem}[{\bf \citet[Theorem 3]{thiemann2017strongly}'s fast-rate bound}]
\label{th:thiemann_pac_bayes}
    Consider a loss function $\ell$ with bounded range $[0,1]$ and let $\bQ_W$ be any prior independent of $S$. Then, for every $\beta \in (0,1)$, with probability no smaller than $1-\beta$
    \begin{equation*}
        \label{eq:thiemann_pac_bayes}
        \bE^S \poprisk{W} \leq \inf_{\lambda \in (0,2)} \Bigg \{ \frac{\bE^S \emprisk{W}{S}}{1 - \frac{\lambda}{2}} + \frac{\relent(\bP_W^S \Vert \bQ_W) + \log \frac{\xi(n)}{\beta}}{n \lambda (1- \frac{\lambda}{2})} \Bigg \}
    \end{equation*}
    holds \emph{simultaneously} for every posterior $\bP_W^S$.
\end{theorem}

\begin{theorem}[{\bf \citet[Theorem 1]{rivasplata2019pac}'s mixed-rate bound}]
\label{th:rivasplata_pac_bayes}
    Consider a loss function $\ell$ with bounded range $[0,1]$ and let $\bQ_W$ be any prior independent of $S$. Then, for every  $\beta \in (0,1)$, with probability no smaller than $1-\beta$
    \begin{align*}
        \label{eq:rivasplata_pac_bayes}
        \bE^S \poprisk{W} \leq &\bE^S \emprisk{W}{S} + \frac{\relent(\bP_W^S \Vert \bQ_W) + \log \frac{\xi(n)}{\beta}}{n} \nonumber \\ &+ \sqrt{2 \bE^S \emprisk{W}{S} \cdot \frac{\relent(\bP_W^S \Vert \bQ_W) + \log \frac{\xi(n)}{\beta}}{n } + \Bigg[\frac{\relent(\bP_W^S \Vert \bQ_W) + \log \frac{\xi(n)}{\beta}}{n } \Bigg]^2}
    \end{align*}
    holds \emph{simultaneously} for every posterior $\bP_W^S$.
\end{theorem}

Originally, \citet{rivasplata2019pac} present their bound in a different form, but this form shows explicitly the combination of a \emph{fast-rate} term and a \emph{slow-rate} term. Moreover, this form makes it easy to see that the bound is tighter than \citep[Equation (3)]{tolstikhin2013pac} as their bound can be recovered using the inequality $\sqrt{x + y} \leq \sqrt{x} + \sqrt{y}$.

\subsection{From Seeger--Langford to an improved Catoni and new fast and mixed-rate bounds}
\label{subsec:seeger_langford_to_catoni_and_new_fast_and_mixed_rate_bounds}

As mentioned previously, both the Seeger--Langford~\citep{langford2001bounds, seeger2002pac} and \citet{catoni2007pac}'s bounds are known to be very tight (see \citet{foong2021tight}) and are often used when a numerical certificate of the population risk is needed~\citep{dziugaite2017computing,perez2021tighter, lotfi2022pac}.  Below, we show that a strengthened version of \citet{catoni2007pac}'s bound that holds \emph{simultaneously} for all $\lambda > 0$ can be obtained from the Seeger--Langford~\citep{langford2001bounds,seeger2002pac} bound at the small cost of $\log \xi(n)$ in the dependence-confidence term. The idea behind the proof is to apply the \citet{donsker1975asymptotic}'s lemma to $\relentber(\bE^S \emprisk{W}{S} \Vert \bE^S \poprisk{W})$. This was also observed by~\citet[Proposition 2.1]{germain2009pac} and proved with different techniques than ours by~\citep[Chapter 20]{catoni2015pac} and~\citet[Lemmata E1 and E2]{foong2021tight}, although it was not stated explicitly as a PAC-Bayes bound.

\begin{theorem}
\label{th:catoni_pac_bayes_uniform}
    Consider a loss function $\ell$ with bounded range $[0,1]$ and let $\bQ_W$ be any prior independent of $S$. Then, for every $\beta \in (0,1)$, with probability no smaller than $1-\beta$
    \begin{equation}
        \label{eq:catoni_pac_bayes_uniform}
        \bE^S \poprisk{W} \leq \inf_{\lambda > 0} \Bigg \{\frac{1}{1 - e^{- \frac{\lambda}{n}}} \Bigg[1 - e^{- \frac{\lambda \bE^S \emprisk{W}{S}}{n}   - \frac{\relent(\bP_W^S \Vert \bQ_W) + \log \frac{\xi(n)}{\beta}}{n}} \Bigg] \Bigg \}
    \end{equation}
    holds \emph{simultaneously} for every posterior $\bP_W^S$.
\end{theorem}

\begin{proof}
    Consider the Seeger--Langford bound from \Cref{th:seeger_langford_pac_bayes}. Applying the \citet[Lemma 2.1]{donsker1975asymptotic}'s variational representation from~\Cref{lemma:dv-var-rep} to the left hand side of~\eqref{eq:seeger_langford_pac_bayes} results in
    \begin{align*}
        \relentber(\bE^S \emprisk{W}{S} \Vert \bE^S \poprisk{W} ) = \sup_{g_0, g_1 \in (-\infty, \infty)} \bigg \{ \bE^S &\emprisk{W}{S} g_1 + (1 - \bE^S \emprisk{W}{S}) g_0 \\&- \log \Big[ \bE^S \poprisk{W} e^{g_1} + (1 - \bE^S \poprisk{W})e^{g_0} \Big]  \bigg \},
    \end{align*}
    where we defined $g_0 \coloneqq g(0)$ and $g_1 \coloneqq g(1)$.
    Re-arranging the terms and plugging them into~\eqref{eq:seeger_langford_pac_bayes} states that with probability no smaller than $1 - \beta$
    \begin{align*}
        \sup_{g_0, g_1 \in (-\infty, \infty)}  \bigg \{ g_0 + \bE^S \emprisk{W}{S} (g_1 - g_0) - \log \Big[e^{g_0} + \bE^S \poprisk{W} (e^{g_1} - e^{g_0}) \Big]
        \bigg \} \leq \mathfrak{C}_{n, \nicefrac{\beta}{\xi(n)}, S}
    \end{align*}
    holds \emph{simultaneously} for every posterior $\bP_W^S$.
    Note that, similarly to \citet{thiemann2017strongly}'s result from \Cref{th:thiemann_pac_bayes}, the bound holds \emph{simultaneously} for all values of the parameters $g_0$ and $g_1$, and therefore these parameters can be chosen adaptively, that is, different values of $g_0$ and $g_1$ can be chosen for different realizations of the training set $s$. Therefore, with probability no smaller than $1 - \beta$
        \begin{equation*}
             -(g_0 - g_1) \bE^S \emprisk{W}{S} - \log \Big(1 - \big(1 - e^{-(g_0-g_1)}\big) \bE^S \poprisk{W} \Big)  \leq \mathfrak{C}_{n, \nicefrac{\beta}{\xi(n)}, S}
        \end{equation*}
    \emph{simultaneously} for every posterior $\bP_W^S$ and all $g_0$ and $g_1$ in $\bR$. Letting $\lambda \coloneqq n (g_0 - g_1)$ and re-arranging the terms in the equation it follows that, with probability no smaller than $1  - \beta$
    \begin{equation*}
        \bE^S \poprisk{W} \leq \frac{1}{1 - e^{-\frac{\lambda}{n}}} \Bigg[1 - e^{- \frac{\lambda \bE^S \emprisk{W}{S}}{n}    - \frac{\relent(\bP_W^S \Vert \bQ_W) + \log \frac{\xi(n)}{\beta}}{n}} \Bigg]
    \end{equation*}
    \emph{simultaneously} for every posterior $\bP_W^S$ and all $\lambda > 0$. The restriction to $\lambda > 0$ instead of $\lambda \in \bR$ comes from the fact that if $\lambda < 0$, then the resulting inequality is a lower bound instead of an upper bound. 
\end{proof}

The bound in \Cref{th:catoni_pac_bayes_uniform} is an explicit expression of the Seeger--Langford bound~\citep{langford2001bounds,seeger2002pac} in terms of $\bE^S \emprisk{W}{S}$ and $\relent(\bP_W^S \Vert \bQ_W)$. Compared to \citet{catoni2007pac}'s \Cref{th:catoni_pac_bayes}, this bound holds \emph{simultaneously} for all $\lambda>0$, making it useful for finding numerical population risk certificates without the need to pay an extra price for the parameter search. It also allows for an iterative procedure for obtaining a good posterior by updating the posterior $\bP_W^S$ and parameter $\lambda$ alternately. We note that, contrary to the statment from the Seeger--Langford bound in \Cref{th:seeger_langford_pac_bayes}, this statment tells us that the optimal posterior is given by the Gibbs distribution $\bP_W^S(w) \propto \bQ_W(w) \cdot e^{- \lambda \emprisk{w}{S}}$. However, finding the global optimum for the parameter $\lambda$ is tedious, and the function is not convex in that parameter.

We may %
massage \Cref{th:catoni_pac_bayes_uniform} to obtain a simpler, more interpretable fast-rate bound.

\begin{theorem}
\label{th:fast_rate_bound_strong}
    Consider a loss function $\ell$ with bounded range $[0,1]$ and let $\bQ_W$ be any prior independent of $S$. Then, for every $\beta \in (0,1)$, with probability no smaller than $1-\beta$
    \begin{equation}
        \label{eq:fast_rate_bound_strong}
        \bE^S \poprisk{W} \leq \inf_{\substack{\gamma > 1 \\c \in (0,1]}} \Bigg \{ c \gamma \log \Big(\frac{\gamma}{\gamma - 1} \Big) \cdot \bE^S \emprisk{W}{S} + c \gamma \cdot \frac{\relent(\bP_W^S \Vert \bQ_W) + \log \frac{\xi(n)}{\beta}}{n} + \kappa(c) \gamma  \Bigg \}
    \end{equation}
    holds \emph{simultaneously} for every posterior $\bP_W^S$,
    where $\kappa(c) \coloneqq 1 - c(1 - \ln c) $.
\end{theorem}

\begin{proof}
The bound follows by noting that the function $1 - e^{-x}$ is a non-decreasing, concave, continuous function for all $x > 0$ and therefore can be upper bounded by its envelope, that is, $1 - e^{-x} = \inf_{a > 0} \{ e^{-a} x + 1 - e^{-a} (1 + a) \}$. Using the envelope in~\eqref{eq:catoni_pac_bayes_uniform}, the changes of variable $a \coloneqq \log \nicefrac{1}{c}$ and $\lambda \coloneqq n \log \nicefrac{\gamma}{(\gamma-1)}$, and noting that $c = e^{-a} \in (0,1]$ and $\gamma = (1 - e^{- \frac{\lambda}{n}})^{-1} > 1$ completes the proof.
\end{proof}

The parameter $\gamma$ controls the influence of the empirical risk compared to the normalized dependence-confidence: if the empirical risk is large relative to the normalized dependence-confidence, then $\gamma$ is larger and the normalized dependence-confidence coefficient increases, if instead the empirical risk is small or even close to interpolation, then $\gamma$ is close to $1$ and the empirical risk coefficient increases. In particular, for a fixed value of $c$, the optimal value of $\gamma$ is
\begin{align*}
	\gamma &= 1 + \left[ -1 - \mathtt{W} \left( - \exp \left(-1 - \frac{c\cdot \frac{\relent(\bP_W^S \Vert \bQ_W) + \log \frac{\xi(n)}{\beta}}{n} + \kappa(c)}{c \cdot \bE \big[ \emprisk{W}{S} \big]} \right) \right) \right]^{-1} \\
	&\approx 1 + \left[ \sqrt{2 \cdot \frac{c\cdot \frac{\relent(\bP_W^S \Vert \bQ_W) + \log \frac{\xi(n)}{\beta}}{n} + \kappa(c)}{c \cdot \bE \big[ \emprisk{W}{S} \big]}} +  \frac{5}{6} \cdot \frac{c\cdot \frac{\relent(\bP_W^S \Vert \bQ_W) + \log \frac{\xi(n)}{\beta}}{n} + \kappa(c)}{c \cdot \bE \big[ \emprisk{W}{S} \big]}\right]^{-1},
\end{align*}
where we considered the Lambert W function $\mathtt{W}$ and \citet{chatzigeorgiou2013bounds}'s approximation of the $-1$ branch.

The parameter $c \in (0,1]$ controls how much weight is given to the empirical risk and normalized dependence-confidence terms compared to a bias. For larger values of the empirical risk and the normalized dependence-confidence term, the value of $c$ is small, decreasing their contribution to the bound and increasing the contribution of the bias $\kappa(c) \in [0,1)$. If the empirical risk and the normalized dependence-confidence term are smaller, then the value of $c$ approaches $1$, where the contribution of these two terms is only controlled by $\gamma$ and the bias is $0$. In fact, a weaker version of~\Cref{th:fast_rate_bound_strong} can be obtained considering this small empirical risk and small normalized dependence-confidence regime by letting $c = 1$.

\begin{corollary}
\label[corollary]{cor:fast_rate_bound}
    Consider a loss function $\ell$ with bounded range $[0,1]$ and let $\bQ_W$ be any prior independent of $S$. Then, for every $\beta \in (0,1)$, with probability no smaller than $1-\beta$
    \begin{equation}
        \label{eq:fast_rate_bound}
        \bE^S \poprisk{W} \leq \inf_{\gamma > 1} \Bigg \{ \gamma \log \Big(\frac{\gamma}{\gamma - 1} \Big) \cdot \bE^S \emprisk{W}{S} + \gamma \cdot \frac{\relent(\bP_W^S \Vert \bQ_W) + \log \frac{\xi(n)}{\beta}}{n}  \Bigg \}
    \end{equation}
    holds \emph{simultaneously} for every posterior $\bP_W^S$.
\end{corollary}

Interestingly, this result can also be obtained using the variational representation of the relative entropy based on $f$-divergences~\citep[Example 7.5]{polyanskiy2022lecture} as shown in \Cref{app:alternative_proof_fast_rate}. That is, the Seeger--Langford bound (\Cref{th:seeger_langford_pac_bayes}), the strengthed Catoni's bound (\Cref{th:catoni_pac_bayes_uniform}), and this fast-rate bound (\Cref{th:fast_rate_bound_strong}) are equally tight. This is important since it means that the Seeger--Langford bound~\citep{langford2001bounds, seeger2002pac} can be exactly described with a linear combination of the empirical risk and the dependence-confidence term, where the coefficients of this combination and the bias vary depending on the data realization. This could have been hypothesized observing the derivatives of the Seeger--Langford bound~\citep{langford2001bounds,seeger2002pac} from \citet[Appendix A]{reeb2018learning}, and a proof is now available. 
Furthermore, the optimal posterior of this bound is given by the Gibbs distribution $\bP_W^S(w) \propto \bQ_W(w) \cdot e^{-n \log \big(\frac{\gamma}{\gamma -1}\big) \emprisk{w}{S}}$, where the value of $\gamma$ depends on the dataset realization $s$.%

The bound improves upon \citet{thiemann2017strongly}'s \Cref{th:thiemann_pac_bayes} as it is tighter for all values of the empirical risk and the dependency measure (see \Cref{subapp:comparison_fast_rate_bounds}). For instance, the value $\lambda = 1$ minimizes the multiplicative factor in the dependence-confidence term in \Cref{th:thiemann_pac_bayes}. Letting $\gamma = 2$ in \Cref{cor:fast_rate_bound} matches this factor and improves the multiplicative factor of the empirical risk from 2 to $2 \log 2 \approx 1.38$. Moreover, if we are in the \emph{realizable setting} and $\bE^S \emprisk{W}{S} = 0$ (that is, we are using an empirical risk minimizer), then letting $\gamma \to 1^+$ in this bound reveals that the fast rate can be achieved with multiplicative factor $1$, clarifying that the dependence-confidence term completely characterizes the population risk in this regime. Note that this is not clear  in \citet{thiemann2017strongly}'s nor \citet{rivasplata2019pac}'s bounds, where the multiplicative factor is 2.

However, substituting the value of the optimal $\gamma$ into~\eqref{eq:fast_rate_bound_strong} or~\eqref{eq:fast_rate_bound} does not produce an interpretable bound. Nonetheless, the bound in \Cref{cor:fast_rate_bound} can be further relaxed to obtain a parameter-free mixed-rate bound that is tighter than \citet{rivasplata2019pac}'s mixed-rate and \citet{thiemann2017strongly}'s fast-rate bounds (see \Cref{subapp:comparison_fast_rate_bounds}).

\begin{theorem}[{\bf mixed-rate bound}]
\label{th:mixed_rate_bound}
    Consider a loss function $\ell$ with bounded range $[0,1]$ and let $\bQ_W$ be any prior independent of $S$. Then, for every $\beta \in (0,1)$, with probability no smaller than $1-\beta$
    \begin{equation}
        \label{eq:mixed_rate_bound}
        \bE^S \poprisk{W} \leq  \bE^S \emprisk{W}{S} +  \frac{\relent(\bP_W^S \Vert \bQ_W) + \log \frac{\xi(n)}{\beta}}{n}  + \sqrt{2 \bE^S \emprisk{W}{S} \cdot \frac{\relent(\bP_W^S \Vert \bQ_W) + \log \frac{\xi(n)}{\beta}}{n}}
    \end{equation}
    holds \emph{simultaneously} for every posterior $\bP_W^S$.
\end{theorem}

\begin{proof}
    Using the inequality $\log x \leq \frac{1}{2} (x - \nicefrac{1}{x})$ for all $x \geq 1$ in~\eqref{eq:fast_rate_bound} establishes that  with probability no smaller than $1 - \beta$
    \begin{equation}
        \label{eq:fast_rate_bound_relaxed}
        \bE^S \poprisk{W} \leq \inf_{\gamma > 1} \Bigg \{ \frac{1}{2} \cdot \frac{2 \gamma - 1}{\gamma - 1} \cdot \bE^S \emprisk{W}{S} + \gamma \cdot \frac{\relent(\bP_W^S \Vert \bQ_W) + \log \frac{\xi(n)}{\beta}}{n}  \Bigg \}.
    \end{equation}
    If $\bE^S \emprisk{W}{S} > 0$, the optimal $\gamma$, which we recall can now be chosen adaptively as the bound holds simultaneously for all $\gamma > 1$, is $\gamma = 1 + \big( \nicefrac{\bE^S \emprisk{W}{S}}{2 \mathfrak{C}_{n, \nicefrac{\beta}{\xi(n)}, S}}\big)^{\nicefrac{1}{2}}$%
    . Substituting this parameter yields~\eqref{eq:mixed_rate_bound}. If $\bE^S \emprisk{W}{S} = 0$, letting $\gamma \to 1^+$ is optimal, which also agrees with the bound in~\eqref{eq:mixed_rate_bound}.
\end{proof}

The mixed-rate bound presented in \Cref{th:mixed_rate_bound} provides a deeper insight into the relationship between the population risk, the empirical risk, and the dependence-confidence term. The bound grows linearly with both the empirical risk and the dependence-confidence term, with a correction term that reflects their interaction. Importantly, the bound is symmetric in these two terms, giving them equal importance. This may be beneficial for methods using PAC-Bayes bounds to optimize the posterior, such as PAC-Bayes with backprop~\citep{rivasplata2019pac, perez2021tighter}, where using the bound from \citet{thiemann2017strongly} or \citet{rivasplata2019pac} alone may cause the algorithm to disregard posteriors farther from the prior but that achieve lower population risk (see \Cref{subapp:pbb}).

\section{PAC-Bayes bounds beyond bounded losses}
\label{sec:pac_bayes_beyond_bounded_losses}

This section is divided into two parts. In \Cref{subsec:losses_more_general_tail_behaviors}, we describe what a ``more general tail behavior'' means and review existing approaches to address this problem. In \Cref{subsec:pac_bayes_bound_bounded_cgf_or_bounded_2nd_moment}, we present a PAC-Bayes analogue to Chernoff's inequality and a bound only requiring that the loss has a bounded second moment. Both bounds are obtained using a new technique for optimizing parameters that need to be selected \emph{before} the draw of the data, but for which optimal solution depends on the data realization. 

\subsection{What are losses with more general tail behaviors?}
\label{subsec:losses_more_general_tail_behaviors}

Probably, the most natural extension of a loss with a bounded range is a subgaussian loss~\citep[Chapter 2]{wainwright2019high}. The loss $\ell(w,Z)$ is \emph{$\sigma^2$-subgaussian} if it is concentrated around its mean with high probability. More precisely, it is \emph{at least} as concentrated as if it were Gaussian with variance $\sigma^2$. That is, for all $w \in \cW$, with probability no smaller than $1 - \beta$ it holds that $\bE \ell(w,Z) \leq \ell(w,Z) + \sqrt{2 \sigma^2 \log \nicefrac{1}{\beta}}$. Notice that this definition of $\sigma^2$-subgaussian is equivalent to definitions expressed in terms of tail probabilities of the random variable $\ell(w,Z)$, or in terms of the moments of this random variable (see~\citet[Theorem 2.6]{wainwright2019high}).

\citet{hellstrom2020generalization} and \citet{guedj2021still} obtained parameterized bounds for subgaussian losses. They also provided a parameter-free version of the bound optimizing the parameter. However, the optimization contained a small mistake, as the parameter needs to be selected \emph{before} the draw of the data, and the value they chose depended on the data realization~\citep[Remark 14]{banerjee2021information}. \citet{hellstrom2021corrections} later resolved this issue to obtain analogous PAC-Bayes bounds employing properties unique to subgaussian random variables~\citep[Theorem 2.6]{wainwright2019high}. \citet{esposito2021generalization} also derived PAC-Bayes bounds for this setting considering different dependency measures.

A step further beyond subgaussian losses are subexponential ones~\citep[Chapter 2]{wainwright2019high}. The loss $\ell(w,Z)$ is subexponential if the probability that it is not concentrated around its mean is exponentially small. That is, for all $w \in \cW$, with probability no smaller than $1-\beta$ it holds that $\bE \ell(w,Z) \leq \ell(w,Z) + \nicefrac{1}{c_2} \log \nicefrac{c_1}{\beta}$ for some $c_1, c_2 > 0$. Note that if a loss is subgaussian, then it is immediately subexponential by letting, for example, $c_1 = e^{\nicefrac{1}{4}}$ and $c_2 = \sqrt{\nicefrac{1}{2\sigma^2}}$.

\citet{catoni2004statistical} derived PAC-Bayes bounds for subexponential losses. However, these bounds are limited to the squared error loss in regression scenarios, where $\cZ = \cX \times \bR$  and the hypothesis $w$ represents the parameters of a regressor $\phi_w: \cX \to \bR$. The analysis assumes that the regressor is finite, that is, that $\lVert \phi_w \rVert_\infty < \infty$ for all $w \in \cW$. Additionally, the derived bounds also rely on a parameter that must be chosen \emph{before} the draw of the data.

These extensions can be generalized with the concept of \emph{cumulant generating function} (CGF). The CGF of a random variable $X$ is defined as $\Lambda_X(\lambda) \coloneqq \log \bE \big[ e^{\lambda (X - \bE X)} \big]$ for all $\lambda \in \bR$. If it exists, it completely characterizes the random variable's distribution, it is convex, continuously differentiable, and $\Lambda_X(0) = \Lambda_X'(0) = 0$~\citep{zhang2006information,banerjee2021information}. We say that the CGF exits if it is bounded in $(-b,b)$ for some $b > 0$.

\begin{definition}[\bf Bounded CGF]
\label[definition]{ass:bounded_cgf}
A loss function $\ell$ is of \emph{bounded CGF} if for all $w \in \cW$, there is a convex and continuously differentiable function $\psi(\lambda)$ defined on $[0, b)$ for some $b \in \bR_+$ such that $\psi(0) = \psi'(0) = 0$ and $\Lambda_{-\ell(w,Z)}(\lambda) \leq \psi(\lambda)$ for all $\lambda \in [0,b)$.
\end{definition}

Notice that we say that a loss has a bounded CGF if the CGF of $- \ell(w,Z)$ is dominated by $\psi$. The reason is that we are interested in bounding from above the random variable $(\poprisk{w} - \ell(w,Z))$, while the CGF of the loss $\ell(w,Z)$ characterizes the tails of the random variable $(\ell(w,Z))- \poprisk{w})$, where we recall that $\poprisk{w} = \bE \ell(w,Z)$.

Under this assumption, the \emph{Cramér--Chernoff} method establishes that with probability no smaller than $1 - \beta$ it holds that $\bE \ell(w,Z) \leq \ell(w,Z) + \psi_*^{-1}(\log \nicefrac{1}{\beta})$~\citep[Section 2.3]{boucheron2013concentration}, where $\psi_*$ is the \emph{convex conjugate} of the function $\psi$ dominating the CGF, and $\psi_*^{-1}$ is its inverse. More details about the convex conjugate and its inverse are given in \Cref{subapp:convex_conjugates_and_inverses}, but as illustrative examples, we have that for bounded losses $\psi_*^{-1}(y) = \sqrt{\nicefrac{y}{2}}$ and for $\sigma^2$-subgaussian losses $\psi_*^{-1}(y) = \sqrt{2\sigma^2 y}$.%

\citet{banerjee2021information} built on \citep{zhang2006information} to prove a parameterized PAC-Bayes bound for losses with bounded CGF. The parameter must be chosen \emph{before} drawing the data, similarly to the bounds in~\citep{hellstrom2020generalization, guedj2021still}, which is a known standing issue in PAC-Bayes literature~\citep{alquier2021user}. A slight extension is given below.

\begin{theorem}[{\bf \citet[Theorem 6]{banerjee2021information}}]
\label{lemma:extension_banerjee}
Consider a loss function $\ell$ with a bounded CGF in the sense of \Cref{ass:bounded_cgf}. Let $\bQ_W$ be any prior independent of $S$. Then, for every $\beta \in (0,1)$ and every $\lambda \in (0,b)$, with probability no smaller than $1-\beta$
\begin{equation}
    \label{eq:pac_bayes_cgf_with_lambda}
    \bE^S \poprisk{W} \leq \bE^S \emprisk{W}{S} + \frac{1}{\lambda} \Bigg[ \frac{\relent(\bP_{W}^{S} \Vert \bQ_W) + \log \frac{1}{\beta}}{n} + \psi(\lambda) \Bigg]
\end{equation}
holds \emph{simultaneously} for every posterior $\bP_W^S$.
\end{theorem}

\noindent{\bf Proof sketch\ }
The proof follows similarly to the first part of \citet[Proposition 3]{guedj2021still}'s proof with the addition of the ideas from \citet{bu2020tightening} and is given in \Cref{subapp:proof_extension_banerjee}.    
\hfill\BlackBox\\[2mm]

A further generalization of the tail behavior of the loss is to consider its moments. The \emph{$m$-th moment} of the loss is defined as $\bE \ell(w,Z)^{m}$. If the CGF exists (that is, it is bounded on an interval around zero), then all the moments are bounded. However, the reverse is not true: for instance, the Pareto distribution of the first kind with parameters $a=3$ and $k=1$ does not have a CGF but its variance is $\nicefrac{3}{4}$~\citep[Chapter 20]{johnson1994continuous}, and the lognormal distribution does not have a CGF but all its moments are finite~\citep[Chapter 14]{johnson1994continuous}~\citep{asmussen2016laplace}.
The smaller the moment assumed to be bounded, the weaker the assumption as if $\bE |\ell(w,Z)^m| < \infty$, then $\bE |\ell(w,Z)^l| < \infty$ for all $l \leq m$.

Instead of directly bounding the population risk considering its tail behavior, \citet{alquier2006transductive} proposed to bound it by studying a \emph{truncated} version of the loss. %
That is, let $\ell^-(w,z) \coloneqq \min \{ \ell(w,z), a \}$ and $\ell^+(w,z) \coloneqq \max \{ \ell(w,z) - a , 0 \}$ for some $a \in \bR$ such that $\ell(w,z) \leq \ell^-(w,z) + \ell^+(w,z)$.
Then, one may find a PAC-Bayes bound for the population risk associated to this truncated loss $\ell^-$ using the standard techniques from \Cref{sec:specialized_pac_bayes_bounds_bounded_losses} and translate that into a PAC-Bayes bound for the real loss accounting for the tail probability $\bE^S \ell^+(w,Z)$. Following this strategy, \citet{alquier2006transductive} proposes bounds that still depend on a parameter that needs to be selected \emph{before} the draw of the data. Moreover, similarly to \citet{catoni2004statistical}, for a regression problem with the square loss and a bounded regressor, \citet{alquier2006transductive} shows the effect of the tail probability $\bP[\cA_w^c]$ is not dominant even if the loss is subexponential.

\citet{alquier2018simpler} also developed PAC-Bayes bounds for losses with heavier tails that sometimes work for non i.i.d. data, although they are not of high probability and consider $f$-divergences as the dependency measure. \citet{holland2019pac} found PAC-Bayes bounds for losses with bounded second and third moments, but consider a different estimate than the empirical risk, and their bounds contain a term that may increase with the number of samples $n$. Finally, \citet{kuzborskij2019efron} and \citet{haddouche2023pacbayes} developed bounds for losses with a bounded second moment. The bound from~\citet{haddouche2023pacbayes} is anytime valid but depends on a parameter that needs to be chosen \emph{before} the draw of the training data.

\citet{haddouche2021pac} developed PAC-Bayes bounds under a different generalization, namely the hypothesis-dependent range (HYPE) condition, that is, that there is a function $\kappa$ with positive range such that $\sup_{z \in \cZ} \ell(w,z) \leq \kappa(w)$ for all hypotheses $w \in \cW$, but their bounds decrease at a slower rate %
than \eqref{eq:mcallester_with_germain_and_mauer_pac_bayes}
when they are restricted to the bounded case. Finally, \citet{chugg2023unified} also proved anytime-valid bounds for bounded CGFs and bounded moments, although their bounds contain parameters that need to be chosen \emph{before} the draw of the training data with other technical conditions. 

\subsection{PAC-Bayes bounds for losses with bounded CGF or bounded second moment}
\label{subsec:pac_bayes_bound_bounded_cgf_or_bounded_2nd_moment}

\citet{mcallester1998some, mcallester1999pac, mcallester2003pac}'s PAC-Bayes bound \eqref{eq:mcallester_with_germain_and_mauer_pac_bayes} can be understood as a PAC-Bayes analogue to Hoeffding's inequality for bounded losses~\citep[Theorem 2.8]{boucheron2013concentration}: the two bounds are equivalent except for the dependency term between the hypothesis $W$ and the training set $S$, namely $\relent(\bP_W^S \Vert \bQ_W)$. If we could optimize the parameter $\lambda$ in \Cref{lemma:extension_banerjee}, we would obtain a PAC-Bayes analogue to Chernoff's inequality for losses with a bounded CGF~\citep[Section 2.2]{boucheron2013concentration}. However, this is not possible since the optimal parameter depends on the data realization but needs to be selected \emph{before} the draw of this data~\citep[Remark 14]{banerjee2021information}.

In the following theorem, we present a technique that allows us to bypass this subtlety for a small penalty of $\log n$ or $\log \log n$. The idea is simple: separate the event space into a finite set of events where the optimization can be performed and then pay the union bound price. This can also be seen as optimizing over the set of parameters $\lambda$ that will yield \emph{almost optimal} bounds, and paying the union bound price for the cardinality of that set. In this  case, the event space is separated using a quantization based on the relative entropy $\relent(\bP_W^S \lVert \bQ_W)$ and noting that the event where $\relent(\bP_{W}^{S} \lVert \bQ_W) > n $ is not interesting as the resulting bound is non-decreasing with $n$ given that event.

\begin{theorem}[\bf PAC-Bayes Chernoff analogue]
\label{th:pac_bayes_chernoff_analogue}
Consider a loss function $\ell$ with a bounded CGF in the sense of \Cref{ass:bounded_cgf}. Let $\bQ_W$ be any prior independent of $S$ and define the event $\cE \coloneqq \{ \relent(\bP_{W}^{S} \Vert \bQ_W) \leq n \}$. Then, for every $\beta \in (0,1)$, with probability no smaller than $1-\beta$
\begin{equation}
\label{eq:pac_bayes_chernoff_analogue}
    \bE^S \poprisk{W} \leq \bI_{\cE} \cdot \Bigg[ \bE^S \emprisk{W}{S} + \psi_{*}^{-1} \bigg( \frac{\relent(\bP_{W}^{S} \Vert \bQ_W) + \log \frac{en}{\beta}}{n} \bigg) \Bigg] + \bI_{\cE^c} \cdot \esssup \bE^S \poprisk{W}
\end{equation}
holds \emph{simultaneously} for every posterior $\bP_W^S$,
where $\bI_\cE$ is the indicator function defined as $\bI_\cE(\omega) = 1$ if $\omega \in \cE$ and $\bI_\cE(\omega) = 0$ otherwise.
\end{theorem}

\begin{proof}
Let $\cB_\lambda$ be the complement of the event in~\eqref{eq:pac_bayes_cgf_with_lambda} such that $\bP[\cB_\lambda] < \beta$ and consider the sub-events $\cE_1 \coloneqq \{ \relent(\bP_W^S \Vert \bQ) \leq 1 \}$ and $\cE_k \coloneqq \{ \lceil \relent(\bP_W^S \Vert \bQ) \rceil = k \}$ for all $k=2, \ldots, n$, which form a covering of the event $\cE$.\footnote{The notation $\lceil x \rceil$ stands for the ceiling of $x$, that is, the nearest integer larger or equal to $x$.} Furthermore, define $\cK \coloneqq \{k \in \bN : 1 \leq k \leq n \textnormal{ and } \bP[\cE_k] > 0 \}$. For all $k \in \cK$, given the event $\cE_k$, with probability no more than $\bP[\cB_\lambda | \cE_k]$, there exists some posterior $\bP_W^S$ such that
\begin{equation}
    \label{eq:pac_bayes_cgf_with_lambda_and_k}
    \bE^S \poprisk{W} > \bE^S \emprisk{W}{S} + \frac{1}{\lambda} \bigg[ \frac{k + \log \frac{1}{\beta}}{n} + \psi(\lambda) \bigg],
\end{equation}
for all $\lambda \in (0, b)$. The right hand side of \eqref{eq:pac_bayes_cgf_with_lambda_and_k} can be minimized with respect to $\lambda$ \emph{independently of the training set $S$}. Let $\cB_{\lambda_k}$ be the event resulting from this minimization and note that $\bP[\cB_{\lambda_k}] \leq \beta$. According to  \citep[Lemma 2.4]{boucheron2013concentration}, this ensures that with probability no more than $\bP[\cB_{\lambda_k} | \cE_k]$, there exists some posterior $\bP_W^S$ such that
\begin{equation}
    \label{eq:pac_bayes_cgf_with_k}
    \bE^S \poprisk{W} > \bE^S \emprisk{W}{S} +\psi_{*}^{-1} \bigg( \frac{k + \log \frac{1}{\beta}}{n} \bigg),
\end{equation}
where $\psi_*$ is the convex conjugate of $\psi$ and where $\psi_*^{-1}$ is a non-decreasing concave function.\footnote{The infimum is not always attained with a particular value $\lambda_k$ of $\lambda$. The details are given in \Cref{subapp:the_optimization_of_lambda}.} Given $\cE_k$, since $k < \relent(\bP_W^S \Vert \bQ_W) + 1$, with probability no larger than $\bP[\cB_{\lambda_k} | \cE_k]$, there exists some posterior $\bP_W^S$ such that
\begin{equation*}
    \bE^S \poprisk{W} > \bE^S \emprisk{W}{S} +\psi_{*}^{-1} \bigg( \frac{\relent(\bP_W^S \Vert \bQ_W) + 1 + \log \frac{1}{\beta}}{n} \bigg).
\end{equation*}
Now, define $\cB'$ as the event stating that there exists some posterior $\bP_W^S$ such that
\begin{equation*}
    \bE^S \poprisk{W} > \bI_{\cE} \cdot \Bigg[ \bE^S \emprisk{W}{S} + \psi_{*}^{-1} \bigg( \frac{\relent(\bP_{W}^{S} \Vert \bQ_W) + \log \frac{e}{\beta}}{n} \bigg) \Bigg] + \bI_{\cE^c} \cdot \esssup \bE^S \poprisk{W}
\end{equation*}
where $\bP[\cB' | \cE_k] \bP[\cE_k] \leq \bP[\cB_{\lambda_k} | \cE_k] \bP[\cE_k] \leq \bP[\cB_{\lambda_k}] \leq \beta$ for all $k \in \cK$ and where $\bP[\cB' \cap \cE^c] = 0$ by the definition of the essential supremum. Therefore, the probability of $\cB'$ is bounded as
\begin{equation*}
    \bP[\cB'] = \sum_{k \in \cK} \bP[\cB' | \cE_k] \bP[\cE_k] + \bP[\cB' \cap \cE^c] < n\beta.
\end{equation*}
Finally, the substitution $\beta \leftarrow \nicefrac{\beta}{n}$ completes the proof.
\end{proof}

For subgaussian losses (and therefore for bounded ones), this recovers \citet{mcallester2003pac}'s  and \citet{hellstrom2021corrections}'s bound rates. For loss functions with heavier tails like \emph{subgamma} and \emph{subexponential}, the rates become a mixture of slow $\sqrt{\mathfrak{C}_{n, \nicefrac{\beta}{en}, S}}$ and fast $\mathfrak{C}_{n, \nicefrac{\beta}{en}, S}$ rates%
.\footnote{Subgamma random variables are nearly subgaussian, but not quite~\citep[Section 2.4]{boucheron2013concentration}.} Please, see~\Cref{cor:pac_bayes_chernoff_analogue_particular} in~\Cref{subapp:pac_bayes_different_tail_behaviors} for an explicit formulation and derivation.

We can also employ this technique to obtain a parameter-free PAC-Bayes bound for losses with a bounded second moment optimizing the parameter in the results from \citet[Theorem 2.4]{wang2015pac} or \citet[Theorem 2.1]{haddouche2023pacbayes}. The resulting bound is similar the one of \citet[Corollary 1]{kuzborskij2019efron}, where the average sum of second moments plays the role of the subgaussian parameter in~\Cref{th:pac_bayes_chernoff_analogue}. Moreover, this bound essentially extends and improves upon the result from~\citet[Corollary 1]{alquier2018simpler} given the relationship between the relative entropy and the $\chi^2$ divergence, that is, that $\relent(\bP \Vert \bQ) \leq \log \big(1 + \chi^2(\bP \Vert \bQ)\big) \leq \chi^2(\bP \Vert \bQ)$~\citep[Section 7.6]{polyanskiy2022lecture}. The proof and further details, including bounds for martingale sequences and non-i.i.d. data, are in \Cref{subapp:closed_form_parameter_free_wang}. 

\begin{theorem}[{\bf Parameter-free bound for bounded second moment}]
    \label{th:parameter_free_anytime_valid_bounded_2nd_moment}
    Let $\bQ_W$ be any prior independent of $S$ and define $\xi'(n) \coloneqq 2 e n (n+1)^2 \log(en)$ and the events $\cE_n \coloneqq \{\sigma^2_n \ \relent(\bP_W^S \Vert \bQ_W) \leq n \}$, where $\sigma_n^2 \coloneqq \frac{1}{n} \sum_{i=1}^n \bE^S [ \ell(W,Z_i)^2 + 2 \ell(W,Z')^2 + 1]$ for all $n \in \bN$. Then, for every $\beta \in (0,1)$, with probability no smaller than $1-\beta$
    \begin{align}
        \bE^S &\poprisk{W} \leq \nonumber \\ &\bI_{\cE_n} \cdot \left[ \bE^S \emprisk{W}{S}] + \frac{2}{\sqrt{6}} \cdot \sqrt{  \sigma^2_n \Bigg( \frac{\relent(\bP_W^{S} \Vert \bQ_W) + \log \frac{\xi'(n)}{\beta} }{n}\Bigg)} \right] + \bI_{\cE_n^c} \cdot \esssup  \bE^S \poprisk{W}.
        \label{eq:parameter_free_anytime_valid_bounded_2nd_moment}
    \end{align}
    holds \emph{simultaneously} for every posterior $\bP_W^S$.
\end{theorem}

Note that in the bound from \Cref{th:parameter_free_anytime_valid_bounded_2nd_moment}, the terms $\bE^S \ell(W,Z_i)^2$ are fully empirical and the term $\bE^S \ell(W,Z')^2$ accounts for the assumption that the second moment of the loss is bounded. \Cref{th:parameter_free_anytime_valid_bounded_2nd_moment} is more general than \Cref{th:pac_bayes_chernoff_analogue}, as only the knowledge of one moment is required instead of the knowledge of a function dominating the CGF, which signifies information of all the moments. However, the resulting bound is not always better. For instance, for $\sigma^2$-subgaussian, $(\sigma^2, c)$-subgamma, or $(\sigma^2, c)$-subexponential losses, the parameter $\sigma^2$ that appears in~\Cref{cor:pac_bayes_chernoff_analogue_particular} to~\Cref{th:pac_bayes_chernoff_analogue} (cf. \Cref{subapp:pac_bayes_different_tail_behaviors}) is a proxy for the variance $\bV^S \ell(W,Z') = \bE^S \big( \ell(W,Z') - \bE^S \poprisk{W} \big)^2$ or the \emph{central} second moment of the loss, while $\bE^S \ell(W,Z')^2$ is its \emph{raw} second moment, which can be much larger than the variance because $\bE^S \ell(W,Z')^2 = \bV^S \ell(W,Z') + \big( \bE^S \poprisk{W} \big)^2$.

\subsubsection{Smaller union bound cost}
\label{subsubsec:smaller_union_bound_cost}

Similarly to \citet{langford2001bounds} and \citet{catoni2003pac}, we can pay a multiplicative cost of $e$ to the relative entropy to reduce the union bound cost to $\log  \nicefrac{ (2 + \log n)}{\beta}$. 
For example, for \Cref{th:pac_bayes_chernoff_analogue} the idea is to follow its proof with the events $\cE_k \coloneqq \{ \relent(\bP_W^S \Vert \bQ_W ) \in (e^{k-1}, e^k] \}$ and note that $e^k < e \relent(\bP_W^S \Vert \bQ_W )$ given $\cE_k$. As mentioned by \citet{maurer2004note}, however, these bounds are only useful when the dependency measure $\relent(\bP_W^S \Vert \bQ_W )$ grows slower than logarithmically. This procedure is detailed in \Cref{subapp:pac_bayes_smaller_union_bound_cost}.

\subsubsection{Different or absence of uninteresting events}
\label{subsubsec:event_cutoff}

\Cref{th:pac_bayes_chernoff_analogue,th:parameter_free_anytime_valid_bounded_2nd_moment} consider the events $\{\relent(\bP_W^S \Vert \bQ_W) \leq n\}$ and $\{\sigma^2_n \relent(\bP_W^S \Vert \bQ_W) \leq n\}$ since the complementary events are uninteresting as the bounds become $\Omega(1)$. However, if one is interested in a different event such as $\{\relent(\bP_W^S \Vert \bQ_W) \leq k_{\textnormal{max}}\} $ or $\{ \sigma^2_ n \relent(\bP_W^S \Vert \bQ_W) \leq k_{\textnormal{max}}\}$, then the proofs may be replicated. The resulting bounds are equal to~\eqref{eq:pac_bayes_chernoff_analogue} and~\eqref{eq:parameter_free_anytime_valid_bounded_2nd_moment}, where the factor inside the logarithm will be $\nicefrac{e k_\textnormal{max}}{\beta}$ and $\nicefrac{\xi'(k_\textnormal{max})}{\beta}$ respectively. Some examples would be to choose $k_{\textnormal{max}} = \lceil \log (d n) \rceil$ for parametric models such as the sparse single-index  \citep{alquier2013sparse} and sparse additive \citep{guedj2013pac} models, where $d$ is the dimension of the input data, or to choose $k_{\textnormal{max}} = \lceil \log (d p n) \rceil$ for the noisy $d \times p$ matrix completion problem \citep{mai2015bayesian}. Other examples are given in \Cref{subapp:pac_bayes_different_tail_behaviors,subapp:recovering_randomized_subsample_setting_pac_bayes_bound}.

Imagine that one is interested in a bound like those presented in \Cref{th:pac_bayes_chernoff_analogue,th:parameter_free_anytime_valid_bounded_2nd_moment} and does not consider any event to be uninteresting. This could happen in some regression application where even if $\relent(\bP_W^S \Vert \bQ_W) \geq n$ and the bound is in $\Omega(1)$ the particular value of the bound is necessary. In this case, working in the events' space is still beneficial. The idea is almost the same as before: separate the events' space into a countable set of events where the optimization can be performed and pay the union bound price. The main difference is that each of these events $\cE_k$ will be defined with a different value of $\beta_k$ so that price of the union bound is finite $\sum_{k=1}^\infty \beta_k < \infty$. For instance, applying this approach to \Cref{lemma:extension_banerjee} results in the following theorem, whose proof is in \Cref{subapp:proof_no_cutoff}.

\begin{theorem}
\label{th:pac_bayes_chernoff_analogue_no_cutoff}
Consider a loss function $\ell$ with a bounded CGF in the sense of \Cref{ass:bounded_cgf}. Let $\bQ_W$ be any prior independent of $S$. Then, for every $\beta \in (0,1)$, with probability no smaller than $1-\beta$
\begin{equation*}
    \bE^S \poprisk{W} \leq \bE^S \emprisk{W}{S} + \psi_{*}^{-1} \Bigg( \frac{\relent(\bP_{W}^{S} \Vert \bQ_W) + \log \frac{e\pi^2\big(\relent(\bP_W^S \Vert \bQ_W) + 1\big)^2}{6\beta}}{n} \Bigg)
\end{equation*}
holds \emph{simultaneously} for every posterior $\bP_W^S$.
\end{theorem}

Since $x + \log \nicefrac{e \pi^2 (x+1)^2}{6 \beta}$ is a non-decreasing, concave, continuous function for all $x > 0$, it can be upper bounded by its envelope. That is, $x + \log \nicefrac{e \pi^2 (x+1)^2}{6 \beta} \leq \inf_{a > 0} \big(\frac{a+3}{a+1}\big) x + \log \nicefrac{e\pi^2 (a+1)^2}{6 \beta} - \frac{2a}{a+1}$. Taking $a=19$ leads to the following corollary.

\begin{corollary}
\label{cor:pac_bayes_chernoff_analogue_no_cutoff_linearized}
Consider a loss function $\ell$ with a bounded CGF in the sense of \Cref{ass:bounded_cgf}. Let $\bQ_W$ be any prior independent of $S$. Then, for every $\beta \in (0,1)$, with probability no smaller than $1-\beta$
\begin{equation*}
    \bE^S \poprisk{W} \leq \bE^S \emprisk{W}{S} + \psi_{*}^{-1} \Bigg( \frac{1.1 \relent(\bP_{W}^{S} \Vert \bQ_W) + \log \frac{10 e\pi^2}{\beta}}{n} \Bigg). 
\end{equation*}
holds \emph{simultaneously} for every posterior $\bP_W^S$.
\end{corollary}

\subsubsection{Related work}
\label{subsubsec:related_work_optimization}

A related, but different technique to deal with these optimization problems is given by \citet{langford2001not} and \citet{catoni2003pac} to solve the bounded losses analogue of~\eqref{eq:pac_bayes_cgf_with_k}. They consider the optimization of $\lambda$ over a geometric grid $\cA = \{e^k : k \in \bN\} \cap [1,n]$ at the smaller union bound cost of $\log \nicefrac{(1+ \log n)} \beta $ at the price of a multiplicative constant of $e$. Using
rounding arguments similar to those in the proofs of \Cref{th:pac_bayes_chernoff_analogue,th:parameter_free_anytime_valid_bounded_2nd_moment}, this translates into being able to optimize the parameter $\lambda$ in the region $[1,n]$. This technique generalizes to other countable families $\cA$ with a union bound cost of $\log |\cA|$~\citep[Section 2.1.4]{alquier2021user}. The downfall of this approach compared to the one presented here is that the optimal parameter $\lambda^\star$ is still dependent on the data drawn $S$, the probability parameter $\beta$, and the tail behavior captured either by $\psi_*^{-1}$ or $\sigma_n^2$. It is hence uncertain if the optimal parameter will lie within the set $\cA$ in general, making a parameter-free expression for the bound impossible.

An extension of this technique is given by~\citet{seldin2012pac}. The idea is to construct a countably infinite grid $\cA$ over the parameters' space and then choose a parameter $\lambda$ from that grid. Then, they can give a closed-form solution by studying how far is the bound resulting from plugging the selected parameter from the grid and the optimal parameter. Their technique has been used for a bounded range and bounded variance setting in~\citep{seldin2012pac} and for a bounded empirical variance in~\citep{tolstikhin2013pac}. 

The main difference between these approaches and ours is that they design a grid $\cA$ over the parameters' space and optimize the parameter $\lambda$ in that grid. Then, the tightness of the resulting bound depends on how well that grid was crafted. This grid $\cA$ needs to be designed in a case-to-case basis and it can be cumbersome, see for example \citep[Appendix A]{tolstikhin2013pac}. Moreover, to design the said grid one requires an explicit expression for the optimal parameter. This may not be available in cases such as in \Cref{th:pac_bayes_chernoff_analogue}, where we only know that~\eqref{eq:pac_bayes_cgf_with_k} is the result of the optimization in~\eqref{eq:pac_bayes_cgf_with_lambda_and_k}. On the other hand, we consider a grid over the events' space and find the \emph{best} parameter for each cell (sub-event) in that grid. This gives three main advantages with respect to the previous techniques. First, the grid is the same for any situation, making the technique easier to employ (see \Cref{th:pac_bayes_chernoff_analogue,th:parameter_free_anytime_valid_bounded_2nd_moment,th:pac_bayes_chernoff_analogue_no_cutoff}, \Cref{th:paramter_free_anytime_valid_martingales,th:pac_bayes_chernoff_analogue_loglog}, and \Cref{cor:pac_bayes_chernoff_analogue_particular}). For instance, it would be trivial to recover a result similar to the PAC-Bayes Bernstein analog of \citet[Theorem 8]{seldin2012pac} optimizing the parameter in \citep[Theorem 7]{seldin2012pac} with our approach. Second, to apply the technique we do not need to know the explicit form of the optimal parameter, which may not exist like in \Cref{th:pac_bayes_chernoff_analogue}, we only need that the optimization is possible. Third, if the grid is made with respect to a random variable $X$, the resulting bound will be tight except from a logarithmic term and an offset changing $X$ by $X + 1$. Therefore, discretizing the events' space is essentially equivalent to crafting a subset $\cA'$ of the parameters' space (not necessarily with a grid structure) with the \emph{optimal} parameters for each region without the need to design this subset $\cA'$ in a case-to-case basis. 

Another possibility to deal with $\lambda$ is to integrate it with respect to an analytically integrable probability density with mass concentrated in its maximum. This is the method employed by~\citet{kuzborskij2019efron} and is known as \emph{the method of mixtures}~\citep[Section 2.3]{de2007pseudo}. Unfortunately, this method requires the existence of a canonical pair: two random variables $X$ and $Y$ satisfying that $\bE \exp \big( \lambda X + \nicefrac{\lambda^2 Y^2}{2} \big) \leq 1$ for all $\lambda$ in the domain of optimization~\cite[Equation (2.2)]{de2007pseudo}. This requirement may not necessarily hold in general settings like \Cref{th:pac_bayes_chernoff_analogue}. Moreover, often this method results in the introduction of a new parameter associated with the density used for integration, for example, the variance of a Gaussian as in \citep{kuzborskij2019efron}. Therefore, our proposed approach is still more general while resulting in essentially the same bound when restricted to the case where the method of mixtures can be employed.\footnote{The final bounds are not directly comparable due to differences in the logarithmic terms, but both are of the same order.}

Finally, \citet[Corollary 8]{kakade2008complexity} employed a similar technique to ours to prove a PAC-Bayes bound for bounded losses similar to \citet{mcallester1998some,mcallester1999pac,mcallester2003pac}'s \Cref{eq:mcallester_with_germain_and_mauer_pac_bayes}. However, they did not employ the technique to optimize a parameter. Instead, they found a bound in terms of a threshold $a$ that held for every posterior $\bP_{W}^S$ such that $\relent(\bP_W^S \Vert \bQ) \leq a$. Then, they discretized the set of all posteriors into the sub-classes $\cP_{k} \coloneqq \{ \bP_W^S : 2^{k+1} < \relent(\bP_W^S \Vert \bQ) \leq 2^{k+2} \}$ and applied the union bound to find a uniform result. This technique is usually known as the \emph{peeling device}, \emph{stratification}, or \emph{slicing} in the probability theory and bandits communities~\citep[Section 13.7]{boucheron2013concentration}~\citep[Section 9.1]{lattimore2020bandit}. The similarity with our proof of~\Cref{th:pac_bayes_chernoff_analogue,th:pac_bayes_chernoff_analogue_no_cutoff,th:pac_bayes_chernoff_analogue_loglog} is clear by looking at our design of the events' discretization and their posterior's sub-classes. However, the nature of the two approaches is different: they have a natural constraint, and they discretize the posterior class space and apply the union bound to circumvent that; while we have a parameter which optimal value is data-dependent, we discretize the events' space to find the optimal parameter in a data-independent way, and then we apply the union bound. %
Moreover, this technique is more general, as one can design the sub-events to include basically any random object that depends on the data as showcased in~\Cref{th:parameter_free_anytime_valid_bounded_2nd_moment,th:paramter_free_anytime_valid_martingales}. Nonetheless, one could consider our technique to be essentially equivalent to the peeling device, since both techniques have the same idea and intention behind them.

\subsubsection{Implications to the design of posterior distributions}
\label{subsubsec:implications_learning_posteriors}

We will focus the discussion about which are the implications of having a parameter-free bound with more general assumptions with respect to the design of posterior distributions to \Cref{th:pac_bayes_chernoff_analogue}. The discussion extends to \Cref{th:parameter_free_anytime_valid_bounded_2nd_moment,th:pac_bayes_chernoff_analogue_no_cutoff} and other situations analogously.

The first consideration is that the parameter-free bound in~\eqref{eq:pac_bayes_chernoff_analogue} can always be transformed back into a parametric bound that holds \emph{simultaneously} for all parameters. In the case of \Cref{th:pac_bayes_chernoff_analogue}, employing \Cref{lemma:boucheron} we have that with probability no smaller than $1 - \beta$
\begin{equation*}
    \bE^S \poprisk{W} \leq \bI_{\cE} \cdot \Bigg[ \bE^S \emprisk{W}{S} + \inf_{\lambda \in (0,b)} \bigg \{ \frac{\relent(\bP_{W}^{S} \Vert \bQ_W) + \log \frac{en}{\beta}}{\lambda n} + \frac{\psi(\lambda)} {\lambda} \bigg \} \Bigg] + \bI_{\cE^c} \cdot \esssup \bE^S \poprisk{W}
\end{equation*}
holds \emph{simultaneously} for every posterior $\bP_W^S$.
This relaxation to a familiar structure tells us that the optimal posterior is the Gibbs distribution $\bP_W^S(w) \propto \bQ_W(w) \cdot e^{\lambda n \emprisk{w}{S}}$, where the value of $\lambda$ can now be chosen \emph{adaptively} for each dataset realization $s$.

The second consideration is if we are using some numerical estimation of the posterior using neural networks as with the PAC-Bayes with backprop~\citep{rivasplata2019pac,perez2021tighter} or other similar frameworks~\citep{dziugaite2017computing,lotfi2022pac}. Then, the posterior can be readily estimatd as long as the inverse of the convex conjugate $\psi_*^{-1}$ is a differentiable function.

\section{Anytime-valid PAC-Bayes bounds}
\label{sec:anytime_valid_pac_bayes_bounds}

Recently, some works \citep{chugg2023unified, jang2023tighter, haddouche2023pacbayes} have focused on \emph{anytime valid} (or \emph{time-uniform}) PAC-Bayes bounds, that is, bounds that hold \emph{simultaneously} for all number of samples $n$. Often, their goal is to provide guarantees at every step for online algorithms that are sequential in nature. These bounds are usually rooted in the usage of supermartingales and \citet{ville1939etude}'s extension of Markov's inequality.

Every standard PAC-Bayes bound can be extended to an anytime-valid bound at a union bound cost, even if it does not have a suitable supermartingale structure. For high-probability PAC-Bayes bounds like \Cref{th:pac_bayes_chernoff_analogue,th:parameter_free_anytime_valid_bounded_2nd_moment}, this extension comes at the small cost of adding $2 \log \nicefrac{\pi n}{\sqrt{6}}$ to $\log \nicefrac{e n}{\beta}$ or to $\log \nicefrac{\xi'(n)}{\beta}$ respectively. This ``folklore'' result is formalized below for general probabilistic bounds. Similar uses of the union bound in other settings appear in~\citep{darling1967iterated, robbins1968iterated, kaufmann2016complexity, howard2021time}.

\begin{theorem}[{\bf From standard to anytime-valid bounds}]
\label{th:standard_to_anytime_valid}
    Consider the probability space $(\Omega, \ccF, \bP)$ and let $(\cE_n)_{n=1}^\infty$ be a sequence of event functions such that $\cE_n: (0,1) \to \ccF$. If $\bP[ \cE_n(\beta) ] \geq 1 - \beta$ for all $\beta \in (0,1)$ and all $n \geq 1$, then $\bP[ \cap_{n=1}^\infty \cE_n(\nicefrac{6 \beta}{\pi^2 n^2}) ] \geq 1 - \beta$ for all $\beta \in (0,1)$.
\end{theorem}

\begin{proof}
    We prove the equivalent statement: ``for every $\beta \in (0,1)$, if $\bP[ \cE_n^c (\beta) ] < \beta$ for all $n \geq 1$, then $\bP[ \cup_{n=1}^\infty \cE_n^c (\nicefrac{6\beta}{\pi^2 n^2})] < \beta$''. By the union bound, it follows that $\bP[ \cup_{n=1}^\infty \cE_n^c (\beta_n) ] < \sum_{n=1}^\infty \beta_n$. Let $\beta_n = \nicefrac{\beta}{n^2}$, then $\bP[ \cup_{n=1}^\infty \cE_n^c (\nicefrac{\beta}{n^2}) ] < \nicefrac{\pi^2 \beta}{6}$. The substitution $\beta \leftarrow \nicefrac{6 \beta}{\pi^2}$ completes the proof.
\end{proof}

There are better choices of $\beta_n$ such as $\beta_n = \nicefrac{\beta}{n \log^2(6n)}$~\citep{kaufmann2016complexity}, but all result in essentially the same cost $\cO(\log n)$ for high-probability PAC-Bayes bounds. The main takeaway from this result is that the anytime-valid bounds obtained via supermartingales and \citet{ville1939etude}'s inequality only contribute in shaving-off a log factor for PAC-Bayes high-probability bounds. Hence, their main advantage is in describing online learning situations where the subsequent samples are dependent to each other, which is not inherently captured by statements like \Cref{th:standard_to_anytime_valid}.

\begin{remark}
\label{rem:specialized_bounds_anytime_valid}
\Cref{th:catoni_pac_bayes_uniform,th:fast_rate_bound_strong,th:mixed_rate_bound} and \Cref{cor:fast_rate_bound} follow verbatim as an anytime-valid bound substituting $\log \nicefrac{\xi(n)}{\beta}$ by $\log \nicefrac{\sqrt{\pi(n+1)}}{\beta}$ without needing \Cref{th:standard_to_anytime_valid}. The reason is that these results are derived from the Seeger--Langford bound~\citep{langford2001bounds, seeger2002pac, maurer2004note}, which is extended to an anytime-valid bound at this cost in \citep{jang2023tighter}.
\end{remark}

\section{Conclusion}
\label{sec:Conclusion}

In this paper, we present new high-probability PAC-Bayes bounds. 
For bounded losses, the strengthened version of Catoni's bound (\Cref{th:catoni_pac_bayes_uniform}) provides tighter fast and mixed-rate bounds (\Cref{th:fast_rate_bound_strong,th:mixed_rate_bound} and \Cref{cor:fast_rate_bound}). Moreover, the fast-rate bound is equivalent to the Seeger--Langford bound~\citep{langford2001bounds,seeger2002pac}, helping us to better understand the behavior of this bound and its optimal posterior. Namely, this reveals that the bound is completely characterized by a linear combination of the empirical risk and the dependence-confidence term, and that 
the optimal posterior of the fast-rate bound is a Gibbs distribution with a data-dependent ``temperature''. For more general losses, we introduce two parameter-free bounds using a new technique to optimize parameters in probabilistic bounds: one for losses with bounded CGF and another one for losses with bounded second moment (\Cref{th:pac_bayes_chernoff_analogue,th:parameter_free_anytime_valid_bounded_2nd_moment}).  We also extend all our results to anytime-valid bounds with a technique that can be applied to any existing bound (\Cref{th:standard_to_anytime_valid}).

PAC-Bayes bounds have been proven useful to provide both numerical population risk certificates as well as to understand the sufficient conditions for a problem to be learned~\citep{ambroladze2006tighter, ralaivola2010chromatic, higgs2010pac, seldin2010pac, alquier2013sparse, guedj2013pac, mai2015bayesian, appert2021new, nozawa2019pac, nozawa2020pac, cherief2022pac}. The interpretability of the bounds in \Cref{sec:specialized_pac_bayes_bounds_bounded_losses} and the wider applicability of those in \Cref{sec:pac_bayes_beyond_bounded_losses}, along with their potential extension from \Cref{sec:anytime_valid_pac_bayes_bounds}, can contribute to extend this understanding. However, it is known that there are situations where an algorithm generalizes but the dependency measure (relative entropy) in the PAC-Bayes bounds is large, yielding them vacuous~\citep{bassily2018learners, livni2020limitation, haghifam2022limitations}. Nonetheless, some approaches recently managed to use PAC-Bayes bounds to obtain non-vacuous bounds for neural networks~\citep{dziugaite2017computing, dziugaite2018data, rivasplata2019pac,  perez2021tighter, zhou2019non, lotfi2022pac}. The bounds from \Cref{sec:specialized_pac_bayes_bounds_bounded_losses} can contribute in this front via methods like PAC-Bayes with backprop~\citep{rivasplata2019pac, perez2021tighter} as they are differentiable and tighter than previous bounds of this kind.

Technically, the procedure employed in \Cref{sec:pac_bayes_beyond_bounded_losses}, focusing on discretizing the event space instead of the parameter space, can be of independent interest and useful for developing theory elsewhere. Similarly, \Cref{th:standard_to_anytime_valid} is presented in a general way so it can be seemingly used in other contexts beyond PAC-Bayesian theory.

\acks{First of all, we would like to thank the reviewers and the editor for their comments and suggestions, which helped us improve the paper. 

We are also grateful to Gergely Neu and our fruitful discussions that lead to a cleaner exposition of \Cref{th:pac_bayes_chernoff_analogue} using indicator functions; to Pierre Alquier for the suggestion of improving \Cref{subsubsec:event_cutoff} considering a cut-off appropriate for PAC-Bayes bounds for parametric problems and for pointing us to further PAC-Bayes literature; and to Omar Rivasplata and María Pérez-Ortiz for their help dealing with PAC-Bayes with backprop.

This work was funded, in part, by the Swedish research council under contracts 2019-03606 (Borja Rodríguez-Gálvez and Mikael Skoglund) and 2021-05266 (Ragnar Thobaben).}

\appendix

\section{{Details of \Cref{sec:specialized_pac_bayes_bounds_bounded_losses}}}
\label{sec:details_specialized_pac_bayes_bounds_bounded_losses}   

This section of the appendix is devoted to providing alternative proofs and supplementary context and examples to the results from~\Cref{sec:specialized_pac_bayes_bounds_bounded_losses}.

First of all, to aid with the proofs of \Cref{th:catoni_pac_bayes_uniform} and \Cref{th:fast_rate_bound_strong}, we state the Donsker--Varadhan~\citep{donsker1975asymptotic} and $f$-divergence-based variational representations of the relative entropy~\citep[Example 7.5]{polyanskiy2022lecture}. For the next two lemmata, $\cX$ denotes a measurable space.

\begin{lemma}[Donsker--Varadhan variational representation of the relative entropy]
\label[lemma]{lemma:dv-var-rep}
    Let $\bP$ and $\bQ$ be two probability measures on $\cX$ and $X \sim \bP$ and $Y \sim \bQ$ be two random variables. Further let $\cG$ be the set of functions $g : \cX \to \bR$ such that $\bE\big[e^{g(Y)}] < \infty$. Then,
    \begin{equation*}
    \label{eq:dv-var-rep}
        \relent(\bP \Vert \bQ) = \sup_{g \in \cG} \bigg \{ \bE \big[g(X) \big] - \log \bE \Big[ e^{g(Y)} \Big] \bigg\}.
    \end{equation*}
\end{lemma}

\begin{lemma}[$f$-divergence-based variational representation of the relative entropy]
\label[lemma]{lemma:fdiv-var-rep}
    Let $\bP$ and $\bQ$ be two probability measures on $\cX$ and $X \sim \bP$ and $Y \sim \bQ$ be two random variables. Further let $\cG$ be the set of functions $g : \cX \to \bR$ such that $\bE\big[e^{g(Y)}] < \infty$. Then,
    \begin{equation*}
    \label{eq:fdiv-var-rep}
        \relent(\bP \Vert \bQ) = \sup_{g \in \cG} \bigg \{ \bE \big[g(X) \big] - \bE \Big[ e^{g(Y)} \Big] + 1 \bigg\}.
    \end{equation*}
\end{lemma}

\subsection{{Alternative proof of \Cref{th:fast_rate_bound_strong}}}
\label{app:alternative_proof_fast_rate}

\sloppy As mentioned in \Cref{subsec:seeger_langford_to_catoni_and_new_fast_and_mixed_rate_bounds}, \Cref{th:fast_rate_bound_strong} can be recovered directly from the Seeger--Langford~\citep{langford2001bounds,seeger2002pac} bound from \Cref{th:seeger_langford_pac_bayes} using the variational representation of the relative entropy based on $f$-divergences~\citep[Example 7.5]{polyanskiy2022lecture}.

\noindent{\bf Alternative proof of \Cref{th:fast_rate_bound_strong}\ }
    Similarly to the proof of \Cref{th:catoni_pac_bayes_uniform}, the proof starts considering the Seeger--Langford bound from \Cref{th:seeger_langford_pac_bayes}. Now, applying the variational represenation of the relative entropy based on $f$-divergences from~\Cref{lemma:fdiv-var-rep} to the right hand side of~\eqref{eq:seeger_langford_pac_bayes} results in
    \begin{align*}
        \relentber(\bE^S \emprisk{W}{S} \Vert \poprisk{W} ) 
        = \sup_{g_0, g_1 \in (-\infty, \infty)} \bigg \{ \bE^S &\emprisk{W}{S} g_1 + (1 - \bE^S \emprisk{W}{S}) g_0 \\ &-  \bE^S \poprisk{W} e^{g_1} + (1 - \bE^S \poprisk{W})e^{g_0} + 1  \bigg \},
    \end{align*}
    where we defined $g_0 \coloneqq g(0)$ and $g_1 \coloneqq g(1)$.
    Re-arranging the terms and plugging them into~\eqref{eq:seeger_langford_pac_bayes} states that with probability no smaller than $1 - \beta$
    \begin{align*}
        \sup_{g_0, g_1 \in (-\infty, \infty)}  \bigg \{1+ g_0 + \bE^S \emprisk{W}{S} (g_1 - g_0) - e^{g_0} - \bE^S \poprisk{W} (e^{g_1} - e^{g_0}) 
        \bigg \} \leq \mathfrak{C}_{n, \nicefrac{\beta}{\xi(n)}, S}.
    \end{align*}
    where we recall that $\mathfrak{C}_{n, \beta, S} \coloneqq \frac{1}{n} \big( \relent(\bP_W^S \Vert \bQ_W) + \log \nicefrac{1}{\beta}\big)$. 
    Again, similarly to \citet{thiemann2017strongly}'s result from \Cref{th:thiemann_pac_bayes}, the bound holds uniformly over all realizations of the training set, and thus the parameters $g_0$ and $g_1$ can be chosen adaptively.         For this inequality to be relevant to us, we require that $g_0 \geq g_1$, as otherwise we would obtain a lower bound instead of an upper bound. To simplify the equations, let $\gamma \coloneqq \nicefrac{e^{g_0} }{(e^{g_0} - e^{g_1})} \geq 1 $, which implies that $g_0 - g_1 = \log \big(\gamma / (\gamma - 1)\big)$ and therefore with probability no smaller than $1 - \beta$
        \begin{align*}
        \label{eq:relentber_f_div_after_change_variable}
            1+ g_0  - e^{g_0} - \log \Big( \frac{\gamma}{\gamma - 1} \Big) \bE^S \emprisk{W}{S}  + \gamma^{-1} e^{g_0} \bE^S \poprisk{W} 
        \end{align*}
        \emph{simultaneously} for every posterior $\bP_W^S$ and all $g_0$ and $g_1$ in $\bR$ such that $g_0 \geq g_1$.

        To finalize the proof, note that the optimal value of the parameter $g_0$ is $\log \big( \nicefrac{\gamma}{(\gamma - \bE^S \poprisk{W})} \big)$ and therefore since $\gamma > 1$ and $\bE^S \poprisk{W} \in [0,1]$, then $g_0 \geq 0$. Finally, letting $c \coloneqq e^{-g_0} \in (0,1]$ and re-arranging the terms recovers the bound in the theorem. %
\hfill\BlackBox\\[2mm]

\subsection{Comparison between the fast-rate and mixed-rate bounds}
\label{subapp:comparison_fast_rate_bounds}

Just by inspecting their equations, it is apparent that the proposed mixed-rate bound of  \Cref{th:mixed_rate_bound} is tighter than those from~\citet{tolstikhin2013pac} and \citet{rivasplata2019pac}. However, it is not directly obvious that the presented fast-rate bound of \Cref{th:fast_rate_bound_strong} is tighter than \citet{thiemann2017strongly}'s \Cref{th:thiemann_pac_bayes}. In fact, even \Cref{cor:fast_rate_bound} is tighter than this result.

To show this, we will show the stronger statement that $f_{\textnormal{fr}}(r,c) \leq f_{\textnormal{th}}(r,c)$ for all $r, c \geq 0$, where 
\begin{align*}
    f_{\textnormal{th}}(r,c) = \inf_{\lambda \in  (0,2)} \Bigg \{ \frac{r}{1 - \frac{\lambda}{2}} + \frac{c}{\lambda \big(1 - \frac{\lambda}{2} \big)} \Bigg \} \textnormal{ and } f_{\textnormal{fr}}(r,c) = \inf_{\gamma > 2} \Bigg \{ \gamma \log \Big(\frac{\gamma}{\gamma - 1} \Big) r +  \gamma \ c\Bigg \}.
\end{align*}
If this holds, then \Cref{cor:fast_rate_bound} is tighter than \Cref{th:thiemann_pac_bayes} as enlarging the optimization set in $f_{\textnormal{fr}}(r,c)$ from $\{\gamma>2\}$ to $\{\gamma>1\}$ will only improve the bound.

Note that with the change of variable $\gamma = \big( \lambda (1 - \nicefrac{\lambda}{2}) \big)^{-1}$, if $\lambda \in (0,2)$, then $\gamma > 2$. This way, we may re-write $f_{\textnormal{fr}}$ in terms of a minimization over $\lambda \in (0,2)$
\begin{align*}
    f_{\textnormal{fr}}(r,c) = \inf_{\lambda \in (0,2)} \Bigg \{ \frac{r}{\lambda \big(1 - \frac{\lambda}{2} \big)} \log \frac{2}{(\lambda - 2) \lambda + 2} + \frac{c}{\lambda \big(1 - \frac{\lambda}{2} \big)} \Bigg \}.
\end{align*}

Finally, noting that 
\begin{equation*}
    \frac{1}{\lambda} \log \frac{2}{(\lambda - 2) \lambda + 2} \leq 1
\end{equation*}
for all $\lambda \in (0,2)$ completes the proof.

Similarly, it can also be shown that the mixed-rate bound from \Cref{th:mixed_rate_bound}, which is itself a relaxation of the fast-rate bound of \Cref{cor:fast_rate_bound}, is also tighter than the fast-rate bound from \citep{thiemann2017strongly}. In this case, we will show the stronger statement that $f_{\textnormal{mr}}(r,c) \leq f_{\textnormal{th}}(r,c)$ for all $r, c \geq 0$, where
\begin{equation}
\label{eq:mixed_rate_relaxed}
        f_{\textnormal{mr}}(r,c) = \inf_{\gamma > 2} \Bigg \{ \frac{1}{2} \cdot \frac{2 \gamma - 1}{\gamma-1} r +  \gamma \ c\Bigg \}.
\end{equation}
As above, as \Cref{th:mixed_rate_bound} is the closed-form expression obtained optimizing the equivalent of~\eqref{eq:mixed_rate_relaxed} on the larger set $\{ \gamma > 1 \}$ (that is, optimizing~\eqref{eq:fast_rate_bound_relaxed}),  showing this statement suffices.

Again, letting $\gamma = \big( \lambda (1 - \nicefrac{\lambda}{2}) \big)^{-1}$, if $\lambda \in (0,2)$ allows us to write $f_{\textnormal{mr}}$ in terms of a minimization over $\lambda \in (0,2)$
\begin{align*}
    f_{\textnormal{mr}}(r,c) = \inf_{\lambda \in (0,2)} \Bigg \{ \frac{1}{2} \cdot \frac{\lambda^2 - 2 \lambda + 4}{\lambda^2 - 2 \lambda + 2} r + \frac{c}{\lambda \big(1 - \frac{\lambda}{2} \big)} \Bigg \}.
\end{align*}

Finally, noting that
\begin{equation*}
    \frac{1}{2} \cdot \frac{\lambda^2 - 2 \lambda + 4}{\lambda^2 - 2 \lambda + 2} \leq \frac{1}{1 - \frac{\lambda}{2}}
\end{equation*}
for all $\lambda \in (0,2)$ completes the proof.

\subsection{Example: PAC-Bayes with backprop using the fast and mixed-rate bounds}
\label{subapp:pbb}

In \Cref{sec:specialized_pac_bayes_bounds_bounded_losses}, we mentioned that methods that use PAC-Bayes bounds to optimize the posterior, such as PAC-Bayes with backprop~\citep{rivasplata2019pac,perez2021tighter}, could benefit from using the bounds from \Cref{th:fast_rate_bound_strong,th:mixed_rate_bound}. In this subsection of the appendix, we provide an example showcasing that this is the case. 

The PAC-Bayes with backprop method~\citep{rivasplata2019pac, perez2021tighter} considers a model $\mathfrak{n}_w$ parameterized by $w \in \bR^d$ and a prior distribution $\bQ_W$ over the parameters, for instance $\bQ_W = w_0 + \sigma_0 \cN(0,I_d)$. Then, the parameters are updated using stochastic gradient descent on the objective
\begin{equation*}
    \emprisk{w}{s} + f_{\textnormal{bound}}(w;\bQ_W),
\end{equation*}
where $\emprisk{w}{s}$ is the empirical risk on the training data realization and $f_{\textnormal{bound}}(w;\bQ_W)$ is extracted from a PAC-Bayes bound evaluated on the parameters $w \in \bR^d$ with prior $\bQ_W$. With an appropriate choice of the posterior $\bP_W^S$, the bound function  $f_{\textnormal{bound}}$ is calculable and the said posterior can be constructed, for instance  $\bP_{W}^S(w) = w + \sigma_0 \cN(0, I_d)$. After the iterative procedure is completed, the empirical risk $\bE^{S=s} \emprisk(W,s)$ is bounded using the Seeger--Langford bound~\citep{langford2001bounds,seeger2002pac} with a Monte Carlo estimate of the posterior parameters of $m$ samples with confidence $1-\beta'$, and the population risk is bounded also using the Seeger--Langford bound~\citep{langford2001bounds,seeger2002pac} with the number of training samples $n$ and a confidence $1- \beta$, amounting for a total confidence of $1 - (\beta' + \beta)$. For more details, please check~\citep{rivasplata2019pac, perez2021tighter}.

Using \citet{thiemann2017strongly}'s \Cref{th:thiemann_pac_bayes}, \citet{rivasplata2019pac}'s \Cref{th:rivasplata_pac_bayes}, or the classical \citet{mcallester2003pac}' bound~\eqref{eq:mcallester_with_germain_and_mauer_pac_bayes} as an objective can be harmful since they penalize too harshly the dependence-confidence term dominated by the normalized dependency $\nicefrac{\relent(\bP_W^S \Vert \bQ_W)}{n}$. Hence, SGD steers the parameters towards places too close to the prior, potentially avoiding other posteriors that achieve lower empirical error and have an overall better population risk. In this sense, it makes sense that bounds such as the proposed fast- and mixed-rate from \Cref{th:fast_rate_bound_strong,th:mixed_rate_bound} or the Seeger--Langford bound~\citep{langford2001bounds,seeger2002pac}, with \citet{reeb2018learning}'s gradients, would lead to said posteriors. This is verified in \Cref{table:results} for a convolutional network and the MNIST dataset. For the fast-rate bound from \Cref{th:fast_rate_bound_strong,cor:fast_rate_bound}, at each iteration the approximately optimal $\gamma$ given after the theorem is employed, thus updating the posterior and the parameter alternately. We saw that the approximation of $\gamma$ is good both by comparing the results of the final posterior in \Cref{table:results} and the coefficients of the empirical risk and the dependence-confidence term in \Cref{fig:comparison_coefficients} with those obtained from the Seeger--Langforfd bound~\citep{langford2001bounds,seeger2002pac} with \citet[Appendix A]{reeb2018learning}'s gradients. After a few iterations, once the empirical risk is small and \Cref{cor:fast_rate_bound} is a good approximation of \Cref{th:fast_rate_bound_strong}, the gradients are close to each other.

\begin{remark}
    \citet{lotfi2022pac} obtain even tighter population risk certificates for networks on the MNIST dataset (11.6 \%) considering a compression approach to the PAC Bayes bound from~\citet{catoni2007pac}. Therefore, their results could be tightened further using our strengthened version from~\Cref{th:catoni_pac_bayes_uniform}. Nonetheless, the goal of this example is not to propose a method that obtains state-of-the-art certificates, but to showcase that the tightness of the tractable bounds in \Cref{sec:specialized_pac_bayes_bounds_bounded_losses} can improve methods that employ PAC-Bayes bounds to find a suitable posterior.
\end{remark}

\begin{table}[ht]
  \caption{Population risk certificate, empirical risk, and normalized dependency of the posterior obtained with PAC-Bayes with Backprop~\citep{rivasplata2019pac,perez2021tighter} using Gaussian priors and different objectives. The best risk certificates are highlighted in bold face, and the second best is highlighted in italics. $^{*}$This refers to the normalized dependency $\nicefrac{\relent(\bP_W^S \Vert \bQ_W)}{n}$. $^{**}$The gradients for the Seeger--Langford~\citep{langford2001bounds,seeger2002pac} bound are not calculated from the bound but hard-coded following \citet[Appendix A]{reeb2018learning}.}
  \label{table:results}
  \centering
  \begin{tabular}{lccc}
    \toprule
    Objective & Risk certificate & Empirical risk & Dependency$^*$ \\
    \midrule
    \Cref{th:rivasplata_pac_bayes}~\citep{rivasplata2019pac} & 0.20870 & 0.11372 & 0.03117 \\
    \Cref{th:thiemann_pac_bayes}~\citep{thiemann2017strongly} & 0.21159 &  0.11053 &  0.03526 \\
    \Cref{eq:mcallester_with_germain_and_mauer_pac_bayes}~\citep{mcallester2003pac} & 0.23658 & 0.23658 & 0.02715 \\
    \Cref{cor:fast_rate_bound} [ours] & \textbf{0.17501} & 0.07054 & 0.04649 \\
    \Cref{th:mixed_rate_bound} [ours] & \textit{0.19763} & 0.09214 & 0.04159 \\
    \midrule 
    \Cref{th:seeger_langford_pac_bayes}~(Seeger--Langford bound)$^{**}$ & \textbf{0.16922} & 0.06701 & 0.04594 \\
    \bottomrule
  \end{tabular}
\end{table}

\begin{figure}
    \centering
    \includegraphics[width=0.6\textwidth]{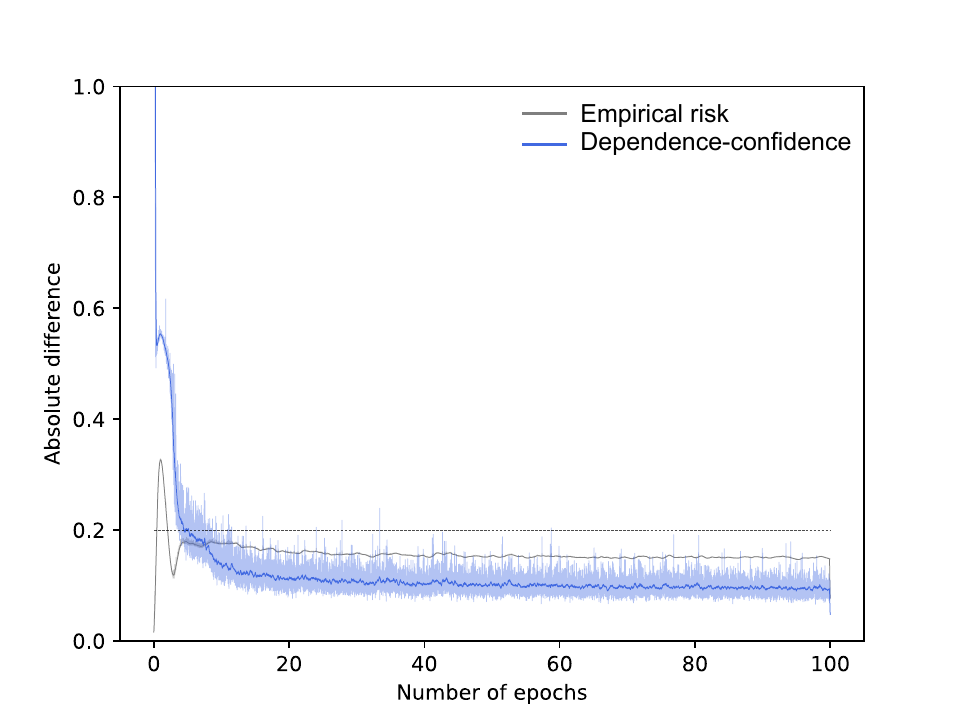}
    \caption{Absolute difference between the coefficients of the empirical risk (gray) and the dependence-confidence term (blue) of the gradients of the Seeger--Langford bound~\citep{langford2001bounds,seeger2002pac} from \citet{reeb2018learning} and the fast-rate bound (\Cref{cor:fast_rate_bound}) using the approximate optimal $\gamma$.}
    \label{fig:comparison_coefficients}
\end{figure}

\subsubsection{Experimental details}

All calculations were performed using the original code from PAC-Bayes with backprop: \url{https://github.com/mperezortiz/PBB}. The file modified to include our bounds and the hard-coded gradients from \citet{reeb2018learning} is \texttt{bounds.py}. The convolutional network architecture consists of two convolutional layers with 32 and 64 filters respectively and a kernel size of 3. The last convolutional layer is followed by a max pooling layer with a kernel size of 2 and two linear layers with 128 and 10 nodes respectively. Between all layers there is a ReLU activation function.

For all experiments, the standard deviation of the prior was $\sigma_0 = 0.1$. The learning rate was 0.01 for all experiments except for \citet{rivasplata2019pac}'s \Cref{th:rivasplata_pac_bayes} objective which was 0.005. The momentum was 0.99 for all objectives except for \citet{thiemann2017strongly}'s \Cref{th:thiemann_pac_bayes} which was 0.95. The number of Monte Carlo samples was $m=150,000$, the minimum probability $p_{\textnormal{min}}$ (see \citep{perez2021tighter} for the details) was $10^{-5}$, and the confidence parameters were $\beta' = 0.01$ and $\beta = 0.025$ respectively. The networks were trained for 100 epochs and a batch size of 250 to mimic the setting in \citep{perez2021tighter}.

To find the hyper-parameters, we used the same grid search as \citet{perez2021tighter}. That is, the standard deviation of the prior was selected over $\{ 0.005, 0.01, 0.02, 0.03, 0.04, 0.05, 0.1 \}$, the learning rate over $\{0.001, 0.005, 0.01 \}$, and the momentum over $\{0.95, 0.99 \}$. Therefore, the confidence parameters were updated to $\beta' \leftarrow \nicefrac{\beta'}{42}$ and $\beta \leftarrow \nicefrac{\beta}{42}$ respectively to comply with the union bound and maintain the guarantees.

All experiments were done on a TESLA V100 with 32GB of memory. Each full run takes approximately 110 hours with most of the time taken on the Monte Carlo sampling for the risk certificates calculation. For  $42 \cdot 5$ runs this amounts to approximately 23,100 hours which is around 32 months.  Since the time was prohibitive for us, the hyper-parameter search was done without the Monte Carlo sampling, where each run took around $25$ minutes amounting to a total of 87.5 hours or less than 4 days. Then, the final certificates were calculated using the full Monte Carlo sampling adding an extra 550 hours or around 23 days. In summary, the total amount of computing was approximately 27 days.

\section{Details of \Cref{sec:pac_bayes_beyond_bounded_losses}}
\label{app:details_pac_bayes_beyond_bounded_losses}

This section of the appendix is devoted to providing alternative proofs and supplementary context and examples to the results from~\Cref{sec:pac_bayes_beyond_bounded_losses}.

\subsection{Convex conjugates and their inverse}
\label{subapp:convex_conjugates_and_inverses}

The convex conjugate is an important concept in convex analysis and in optimization theory. It is also particularly important to derive concentration inequalities through the Cramér--Chernoff method \citep[Section 2.2]{boucheron2013concentration} as shown in \Cref{sec:pac_bayes_beyond_bounded_losses}.

\begin{definition}
\label{def:convex_conjugate}
The \emph{convex conjugate} (or just \emph{conjugate} or \emph{Fenchel-Legendre's dual}) of a function $\psi: \bR \to \bR$ is defined as 
\begin{equation*}
    \psi_*(x) \coloneqq \sup_{\lambda \in \textnormal{dom}(\psi)} \big \{ \lambda x - \psi(\lambda) \big \}.
\end{equation*}
\end{definition}

Specifically, when the function $\psi$ is convex, the convex conjugate is also known as the \emph{Legendre's transform}, and when $\psi$ represents or dominates a CGF as in \Cref{ass:bounded_cgf}, it is known as the \emph{Cramér's transform}. We fall under both of these situations so the particular results for these transforms apply. A particularly important result is the following, which states an expression of the inverse of the convex conjugate of a smooth convex function. This result is used both to obtain the classical Chernoff's inequality~\citep[Section 2.2]{boucheron2013concentration} and its PAC-Bayes analogue from \Cref{th:pac_bayes_chernoff_analogue}.

\begin{lemma}[{\bf \citet[Lemma 2.4]{boucheron2013concentration}}]
\label[lemma]{lemma:boucheron}
Let $\psi$ be a convex and continously differentiable function defined on $[0,b)$ where $0 < b \leq \infty$. Assume that $\psi(0) = \psi'(0) = 0$. Then, the convex conjugate $\psi_*$ is a non-negative convex and non-decreasing function on $[0, \infty)$. Moreover, for every $y \geq 0$, the set $\{ x \geq 0 : \psi_*(x) > y \}$ is non-empty and the generalized inverse of $\psi_*$, defined as $\psi_*^{-1}(y) \coloneqq \inf \{ x \geq 0 : \psi_*(x) > y \}$ can also be written as
\begin{equation*}
    \psi_*^{-1}(y) = \inf_{\lambda \in (0,b)} \bigg \{ \frac{y + \psi(\lambda)}{\lambda} \bigg \}.
\end{equation*}
\end{lemma}

\subsection{Proof of \Cref{lemma:extension_banerjee}}
\label{subapp:proof_extension_banerjee}

Similarly to what we did in~\Cref{sec:details_specialized_pac_bayes_bounds_bounded_losses}, we first introduced \citet[Theorem III on Chapter XI]{gibbs1902elementary}' variational principle, which is a dual formulation to the~\citet{donsker1975asymptotic}~\Cref{lemma:dv-var-rep}. 

\begin{lemma}[Gibbs' variational principle]
    \label[lemma]{lemma:gvp}
    Let $\cX$ be a measurable space, $\bQ$ be a probability measure on $\cX$, and $Y$ be a random variable distributed according to $\bQ$. Further let $g$ be a measurable function on $\cX$ such that $\bE \big[ e^{g(Y)} \big] < \infty$ and $\cP_{\bQ}(\cX)$ be the set of all probability measures $\bP$ on $\cX$ such that $\bP \ll \bQ$. Then, 
	\begin{equation*}
		\log \bE \Big[e^{g(Y)} \Big] = \sup_{\bP \in \cP_\bQ(\cX)} \Big\{ \bE \big[g(X)\big]  - \relent(\bP \Vert \bQ) \Big\}.
	\end{equation*}
\end{lemma}

\noindent{\bf Alternative proof of \Cref{lemma:extension_banerjee}\ } \citet{gibbs1902elementary}' variational principle from~\Cref{lemma:gvp} states that for all measurable functions $g$ such that $e^g$ is $\bQ_W$ integrable
\begin{equation}
    \label{eq:gvp}
    \log \bE^{S} \Big[ e^{g(W')} \Big] = \sup_{\bP_W^{S} \in \cP_\bQ(\cW)} \big\{ \bE^{S} g(W) - \relent(\bP_W^{S} \Vert \bQ_W)  \big\} \textnormal{ a.s.,}
\end{equation}
where $W'$ is distributed according to $\bQ_W$ and where $\cP_\bQ(\bW)$ is the set of all measures on $\cW$ such that $\bP_W^{S} \ll \bQ$ a.s..
In this case, let $g(w;s) \coloneqq \lambda \big( \poprisk{w} - \emprisk{w}{s} \big)$ for some $\lambda \in (0, b/n)$. The first term in the right hand side of~\eqref{eq:gvp} is directly $\lambda \bE^{S} [ \poprisk{W} - \emprisk{W}{S}]$. For the term in the left hand side, one may employ Markov's inequality and Fubini's theorem to see that with probability no smaller than $1-\beta$
\begin{equation*}
    \bE^S \Big[ e^{g(W';S)} \Big] \leq \frac{1}{\beta} \cdot \bE \bigg[ \bE^{W'} \Big[ e^{g(W';S)} \Big] \bigg].
\end{equation*}
Then, since $\Lambda_{-\ell(w,Z)}(\lambda) \leq \psi(\lambda)$ for all $w \in \cW$ it holds that $\Lambda_{\poprisk{w}-\emprisk{w}{S}} \leq n \psi(\nicefrac{\lambda}{n})$. Indeed,
\begin{equation*}
    \log \bE \bigg[ e^{\lambda \big(\poprisk{w} - \emprisk{w}{S} \big)}\bigg] = \log \bE \bigg[ e^{\frac{\lambda}{n} \sum_{i=1}^n \big( \bE \ell(w,Z) - \ell(w,Z_i) \big)}\bigg]  \leq n \psi \Big( \frac{\lambda}{n} \Big).
\end{equation*}
Therefore, with probability no smaller than $1-\beta$
\begin{equation*}
    \log \bE^S \bigg[ e^{\lambda \big(\poprisk{W'} - \emprisk{W'}{S} \big)}\bigg] \leq n \psi \Big( \frac{\lambda}{n} \Big) + \log \frac{1}{\beta}.
\end{equation*}
Combining the results for both terms together, for all $\lambda \in (0, \nicefrac{b}{n})$, with probability larger or equal than $1 - \beta$
\begin{equation*}
     \bE^S \poprisk{W} \leq \bE^S \emprisk{W}{S} + \frac{1}{\lambda} \bigg[ \relent(\bP_W^S \Vert \bQ_W) + \log \frac{1}{\beta} + n \psi \Big( \frac{\lambda}{n}\Big) \bigg]
\end{equation*}
holds \emph{simultaneously} for all $\bP_W^S \in \cP_{\bQ_W}(\cW)$, where the fact that it holds uniformly for all the Markov kernels $\bP_W^S$ comes from the supremum. Since the result holds for all $\lambda \in (0, \nicefrac{b}{n})$, performing the substitution $\lambda \leftarrow n \lambda$ completes the proof.
\hfill\BlackBox\\[2mm]

\subsection{{The optimization of $\lambda$ in \eqref{eq:pac_bayes_cgf_with_lambda_and_k}}}
\label{subapp:the_optimization_of_lambda}

Recall that in the proof of \Cref{th:pac_bayes_chernoff_analogue} we considered the event $\cB_\lambda$ to be the complement of the event in~\eqref{eq:pac_bayes_cgf_with_lambda} such that $\bP[\cB_\lambda] \leq \beta$ for all $\lambda \in (0,b)$. This event is parameterized with $\lambda$ and, given the event $\cE_k = \{ \lceil \relent(\bP_W^S \Vert \bQ_W) \rceil = k \}$, with probability no more than $\bP[\cB_\lambda | \cE_k]$, there exists some posterior $\bP_W^S$ such that
\begin{equation}
    \bE^S \poprisk{W} > \bE^S \emprisk{W}{S} + \frac{1}{\lambda} \bigg[ \frac{k + \log \frac{1}{\beta}}{n} + \psi(\lambda) \bigg]
    \tag{\ref{eq:pac_bayes_cgf_with_lambda_and_k}}
\end{equation}
for all $\lambda \in (0,b)$. Then, after optimizing the parameter $\lambda$ on the right hand side of~\eqref{eq:pac_bayes_cgf_with_lambda_and_k} using \Cref{lemma:boucheron}, %
we considered the event $\cB_{\lambda_k}$ resulting of that optimization such that $\bP[\cB_{\lambda_k}] \leq \beta$.
This notation is imprecise, the infimum in \eqref{eq:pac_bayes_cgf_with_lambda_and_k} is attained either by a $\lambda_k \in (0,b)$ or by letting $\lambda \to b$. It will never be attained when $\lambda \to 0$ as $\psi(0) = 0$ and the term inside the infimum goes to $\infty$ when $\lambda \to 0$. In the case where the infimum is attained by letting $\lambda \to b$, by continuity, the desired inequality~\eqref{eq:pac_bayes_cgf_with_k} still holds and the event described by $\lim_{\lambda \to b} \cB_\lambda$ is still such that $\bP[\lim_{\lambda \to b} \cB_\lambda] \leq \beta$. We hide these details from the main text for clarity of exposition.

\subsection{PAC-Bayes bounds for different loss tail behaviors}
\label{subapp:pac_bayes_different_tail_behaviors}

\Cref{th:pac_bayes_chernoff_analogue} describes different high-probability PAC-Bayes bounds for different tail behaviors of the loss $\ell(w,Z)$. The next corollary collects some of the most common tail behaviors and their resulting PAC-Bayes bound.

\begin{corollary}
\label[corollary]{cor:pac_bayes_chernoff_analogue_particular}
Consider a training set $S$ with $n$ samples. Let $\bQ_W$ be any prior independent of $S$ and define $\mathfrak{C}_{n, \beta, S} \coloneqq \frac{1}{n} \big(\relent(\bP_W^S \Vert \bQ_W) + \log \tfrac{1}{\beta} \big)$ and the event $\cE \coloneqq \{ \relent( \bP_W^S \Vert \bQ_W) \leq n \}$. Then:
\begin{enumerate}
    \item if the loss $\ell$ has a bounded range $[a,b]$, where $-\infty < a \leq b < \infty$, then with probability no smaller than $1-\beta$
    \begin{equation*}
        \bE^S \poprisk{W} \leq  \bE^S \emprisk{W}{S} + \sqrt{(b-a)^2 \mathfrak{C}_{2n, \nicefrac{\beta}{en}, S}}
    \end{equation*}
    holds \emph{simultaneously} for every posterior $\bP_W^S$;

    \item if the loss $\ell(w,Z)$ is $\sigma^2$-subgaussian for all hypotheses $w \in \cW$, then with probability no smaller than $1-\beta$
    \begin{equation*}
        \bE^S \poprisk{W} \leq \bI_\cE \cdot \Bigg[ \bE^S \emprisk{W}{S} + \sqrt{2 \sigma^2 \mathfrak{C}_{n, \nicefrac{\beta}{en}, S}} \Bigg] + \bI_{\cE^c} \cdot \esssup \bE^S \poprisk{W}
    \end{equation*}
    holds \emph{simultaneously} for every posterior $\bP_W^S$;

    \item if the loss $\ell(w,Z)$ is $(\sigma^2, c)$-subgamma for all hypotheses $w \in \cW$, then with probability no smaller than $1-\beta$
    \begin{equation*}
        \bE^S \poprisk{W} \leq \bI_\cE \cdot \Bigg[ \bE^S \emprisk{W}{S} + \sqrt{2 \sigma^2 \mathfrak{C}_{n, \nicefrac{\beta}{en}, S}} + c \ \mathfrak{C}_{n, \nicefrac{\beta}{en}, S} \Bigg] + \bI_{\cE^c} \cdot \esssup \bE^S \poprisk{W}
    \end{equation*}
    holds \emph{simultaneously} for every posterior $\bP_W^S$;

    \item and if $\ell(w,Z)$ is $(\sigma^2, c)$-subexponential\footnote{Here, we are considering the subexponential characterization of random variables from~\citet[Theorem 2.13]{wainwright2019high}, and not the one given by~\citet[Excercise 2.22]{boucheron2013concentration}.} for all hypotheses $w \in \cW$, then with probability no smaller than $1-\beta$
    \begin{align*}
        \bE^S \poprisk{W} \leq &\bI_\cE \cdot \Bigg[ \bE^S \emprisk{W}{S} + \bI_\cF \cdot \sqrt{2 \sigma^2 \mathfrak{C}_{n, \nicefrac{\beta}{en}, S}} + \bI_{\cF^c} \cdot (c+1) \mathfrak{C}_{n, \nicefrac{\beta}{en}, S} \Bigg] \\ 
        &+ \bI_{\cE^c} \cdot \esssup \bE^S \poprisk{W}
    \end{align*}
    holds \emph{simultaneously} for every posterior $\bP_W^S$,
    where $\cF$ is the event $\cF \coloneqq \{ \relent(\bP_W^S \Vert \bQ_W) \leq \frac{n \sigma^2}{2c} - \log \frac{e}{\beta} \}$.
\end{enumerate}

\end{corollary}

\begin{proof}
    We may prove each point individually:
    \begin{itemize}
        \item Point 2 follows by noting that for $\sigma^2$-subgaussian random variables $\psi(\lambda) = \nicefrac{\lambda^2 \sigma^2}{2}$ and therefore $\psi_*^{-1}(y) = \sqrt{2 \sigma^2 y}$.
        \item Point 1 follows by noting that if a random variable is bounded in $[a, b]$, then it is $(b-a)^2/2$-subgaussian. Then, as hinted in~\Cref{subsubsec:event_cutoff}, we may consider the more lenient event $\cE' \coloneqq \{ \relent(\bP_W^S \Vert \bQ_W) \leq 2n \}$ and use that $\esssup \bE^S \poprisk{W} \leq b-a \leq \sqrt{(b-a)^2 \mathfrak{C}_{2n,\nicefrac{\beta}{en},S}}$.
        \item Point 3 follows by noting that for $(\sigma^2, c)$-subgamma random variables $\psi(\lambda) = \nicefrac{\lambda^2 \sigma^2}{2(1-c \lambda)}$ for all $\lambda \in (0, \nicefrac{1}{c})$ and therefore $\psi_*^{-1}(y) = \sqrt{2 \sigma^2 y} + c y$ \citep[Section 2.4]{boucheron2013concentration}.
        \item Finally, Point 4 follows by noting that for $(\sigma^2, c)$-subexponential ranom objects $\psi(\lambda) = \nicefrac{\lambda^2 \sigma^2}{2}$ for all $\lambda \in (0, \nicefrac{1}{c})$ and therefore 
        \begin{equation*}
            \psi_*^{-1}(y) = \begin{cases}
                \sqrt{2 \sigma^2 y} & \textnormal{if } \lambda = \sqrt{\nicefrac{2 y}{\sigma^2}} \leq \nicefrac{1}{c} \\
                c y + \nicefrac{\sigma^2}{2 c^2} & \textnormal{otherwise}
            \end{cases}.
        \end{equation*}
        The condition for the first case may be rewritten as $y \leq \nicefrac{\sigma^2}{2 c^2}$ and similarly the condition for the second case as $y > \nicefrac{\sigma^2}{2 c^2}$. Hence, we have the inequality
        \begin{equation*}
            \psi_*^{-1}(y) \leq \begin{cases}
                \sqrt{2 \sigma^2 y} & \textnormal{if } y \leq \nicefrac{\sigma^2}{2 c^2} \\
                (c + 1) y & \textnormal{otherwise}
            \end{cases}.
        \end{equation*}
        
    \end{itemize}

\end{proof}

\subsection{A parameter-free version of \citet{wang2015pac} and \citet[Theorem 2.1]{haddouche2023pacbayes}'s  PAC-Bayes bound on martingales}
\label{subapp:closed_form_parameter_free_wang}

\citet{wang2015pac} and \citet{haddouche2023pacbayes} investigate the setting where the dataset $S$ is considered to be a sequence $S^* \coloneqq (Z_i)_{i \geq 1}$ such that $Z_i \in \cZ$, but where there is no restriction in the distribution of the samples $Z_i$, that is, every sample $Z_i$ can depend on all the previous ones. For every $n$, they let $S_n \coloneqq (Z_1, \ldots, Z_n)$ be the restriction of $S^*$ to its first $n$ points. Then, they consider the sequence of $\sigma$-algebras $(\ccF_i)_{i \geq 1}$ to be a filtration adapted to $S^*$, for instance $\ccF_i = \sigma(Z_1, \ldots, Z_i)$. Finally, they consider a martingale difference sequence $(X_i(S_i, w))_{i \geq 1}$ indexed by a hypothesis $w \in \cW$ so that $\bE^{\ccF_{i-1}} X_i(S_i,w) = 0$ for all $w \in \cW$. For instance, let $Y_0 = \sum_{i=1}^n \bE \ell(w,Z_i) $ and $Y_i(S_i, w) = \sum_{i=1}^n \bE^{\ccF_i} \ell(w, Z_i)$ for all $i \geq 1$, then $X_i(S_i, w) = Y_i - Y_{i-1}$. Finally, for all $w \in \cW$, they define the martingale $M_n(w) \coloneqq \sum_{i=1}^n X_i(S_i,w)$ and follow \citet{bercu2008exponential} to also define
\begin{equation*}
    [M]_n(w) \coloneqq \sum_{i=1}^n X_i(S_i,w)^2 \textnormal{ and } \langle M \rangle_n(w) \coloneqq \bE^{\ccF_{i-1}} X_i(S_i,w)^2,
\end{equation*}
where $[M]_n(w)$ acts as an empirical variance term and $\langle M \rangle_n(w)$ as its theoretical counterpart~\citep{haddouche2023pacbayes}.

Then, their main anytime-valid bound for martingales is the following.

\begin{theorem}[{\bf \citet[Theorem 2.4]{wang2015pac}}]
    \label{th:wang_main_th}
    Let $\bQ_W$ be any prior independent of $S_n$ and $(M_n(w))_{n \geq 1}$ be any collection of martingales indexed by $w \in \cW$. Then, for all $\lambda > 0$, all $\beta \in (0,1)$, and \emph{simultaneously} for all $n \geq 1$, with probability no smaller than $1-\beta$
    \begin{equation}
        \label{eq:wang_main_th}
        | \bE^{S_n} M_n(W) | \leq \frac{\relent(\bP_W^{S_n} \Vert \bQ_W) + \log \frac{2(n+1)^2}{\beta}}{\lambda} + \frac{\lambda}{6} \cdot \bE^{S_n} \big[ [M]_n(W) + 2 \langle M \rangle_n(W) \big]
    \end{equation}
    holds \emph{simultaneously} for every posterior $\bP_W^{S_n}$.
\end{theorem}

\begin{theorem}[{\bf \citet[Theorem 2.1]{haddouche2023pacbayes}}]
    \label{th:haddouche_main_th}
    Let $\bQ_W$ be any prior independent of $S_n$ and $(M_n(w))_{n \geq 1}$ be any collection of martingales indexed by $w \in \cW$. Then, for all $\lambda > 0$, all $\beta \in (0,1)$, and \emph{simultaneously} for all $n \geq 1$, with probability no smaller than $1-\beta$
    \begin{equation*}
        \label{eq:haddouche_main_th}
        | \bE^{S_n} M_n(W) | \leq \frac{\relent(\bP_W^{S_n} \Vert \bQ_W) + \log \frac{2}{\beta}}{\lambda} + \frac{\lambda}{2} \cdot \bE^{S_n} \big[ [M]_n(W) + \langle M \rangle_n(W) \big]
    \end{equation*}
    holds \emph{simultaneously} for every posterior $\bP_W^{S_n}$.
\end{theorem}

In what follows, we will focus on the result from \citet{wang2015pac} as it has the smaller constants. Taking a closer look at \Cref{th:wang_main_th}, we realize it has a similar shape to \Cref{lemma:extension_banerjee} for the particular case when the loss is subgaussian, where the role of the subgaussian parameter is taken by the sum of the ``variance'' terms $[M]_n(W) + 2 \langle M \rangle_n(W)$. Therefore, it appears we may directly employ the technique to derive the Chernoff analogue from the proof of \Cref{th:pac_bayes_chernoff_analogue}. However, one needs to take into account the fact that the ``optimal'' parameter $\lambda$ now depends on this ``variance'' terms, which are also dependent on the training set $S_n$ and on the number of samples $n$.

To optimize the bound from \Cref{th:wang_main_th} we will then proceed in two steps. The first step is to optimize the parameter $\lambda$ for a \emph{fixed} number of samples $n$ in a similar fashion to \Cref{th:pac_bayes_chernoff_analogue}, which results in \Cref{th:paramter_free_anytime_valid_martingales}. Then, the second step is to extend this result to an anytime valid bound using \Cref{th:standard_to_anytime_valid} at a cost in the dependence-confidence term of $\cO(\nicefrac{\log n}{n})$.

For the first step,  define the event $\cB_{n,\lambda}$ as the complement of the event in~\eqref{eq:wang_main_th} for a \emph{fixed} number of samples $n$. Then, we can proceed similarly to the proof of \Cref{th:pac_bayes_chernoff_analogue} noticing that, for each number of samples $n$, the complement of the event 
\begin{equation*}
    \cE_n \coloneqq \Big \{ \bE^{S_n} \big[ [M]_n(W) + 2 \langle M \rangle_n(W) \big] \relent(\bP_W^S \Vert \bW_W) \leq n^2 \Big\}
\end{equation*}
is uninteresting as the bound is non-vanishing given $\cE_n^c$. This produces the following PAC-Bayes bound for a fixed number of samples $n$.%

\begin{theorem}[{\bf Parameter-free bound on martingales}]
    \label{th:paramter_free_anytime_valid_martingales}
    Let $\bQ_W$ be any prior independent of $S_n$ and $(M_n(w))_{n \geq 1}$ be any collection of martingales indexed by $w \in \cW$. Further, define $\xi'(n) \coloneqq 2en(n+1)^2 \log(en) \leq 2e(n+1)^3$. Then, for every $\beta \in (0,1)$, with probability no smaller than $1-\beta$
    \begin{align*}
        | \bE^{S_n} M_n(W) | \leq &\bI_{\cE_n} \cdot \frac{2}{\sqrt{6}} \cdot \sqrt{ \bE^{S_n} \big[ [M]_n(W) + 2 \langle M \rangle_n(W) + 1 \big]\Big( \relent(\bP_W^{S_n} \Vert \bQ_W) + \log \frac{\xi'(n)}{\beta_n} \Big)} \\
        &+ \bI_{\cE_n^c} \esssup | \bE^{S_n} M_n(W) |
    \end{align*}
    holds \emph{simultaneously} for all posteriors $\bP_W^{S_n}$,
    where $\cE_n$ is the event $\cE_n \coloneqq \Big\{ \bE^{S_n} \big[ [M]_n(W) + 2 \langle M \rangle_n(W) \big] \relent(\bP_W^S \Vert \bW_W) \leq n^2 \Big\}$.
\end{theorem}

\begin{proof}
    Consider a fixed number of samples $n$. Let $\cB_{n,\lambda}$ be the complement of the event in~\eqref{eq:wang_main_th} such that $\bP[\cB_{n,\lambda}] < \beta$ and consider the sub-events 
    \begin{align*}
        \cE_{n,1,l} &\coloneqq \left \{ \relent(\bP_W^{S_n} \Vert \bQ_W) \leq 1 \textnormal{ and } \left \lceil \bE^{S_n} \big[ [M]_n(W) + 2 \langle M \rangle_n(W) \big] \right \rceil = l \right \}, \\
        \cE_{n,k, 1} &\coloneqq \left \{ \left \lceil \relent(\bP_W^{S_n} \Vert \bQ_W) \right \rceil = k \textnormal{ and }  \bE^{S_n} \big[ [M]_n(W) + 2 \langle M \rangle_n(W) \big] \leq 1 \right \}, \textnormal{ and} \\
        \cE_{n,k,l} &\coloneqq \left \{ \left \lceil \relent(\bP_W^{S_n} \Vert \bQ_W) \right \rceil = k  \textnormal{ and } \left \lceil \bE^{S_n} \big[ [M]_n(W) + 2 \langle M \rangle_n(W) \big] \right \rceil = l \right \},
    \end{align*}
    for all $k, l = 2, \cdots, n^2$ such that $k l \leq n^2$, which form a covering of $\cE_n$. Furthermore, define $\cK \coloneqq \{ (k,l) : 1 \leq k l \leq n^2 \textnormal{ and } \bP[\cE_{n,k,l}] > 0 \}$. For all $(k,l) \in \cK$, given the event $\cE_{n,k,l}$, with probability no more than $\bP[\cB_{n,\lambda} | \cE_{n,k,l}]$, there exists some posterior $\bP_W^{S_n}$ such that
    \begin{equation}
        \label{eq:pac_bayes_martingale_with_lambda}
        | \bE^{S_n} M_n(W) | > \frac{k + \log \frac{2(n+1)^2}{\beta}}{\lambda} + \frac{\lambda}{6} \cdot l.
    \end{equation}
    for all $\lambda \in (0, b)$. The parameter that optimizes the right hand side of~\eqref{eq:pac_bayes_martingale_with_lambda} is
    \begin{equation*}
        \lambda = \lambda_{k,l} = \sqrt{\frac{6}{l} \Big(k + \log \frac{2(n+1)^2}{\beta} \Big)}.
    \end{equation*}
    Substituting the optimal $\lambda_{k,l}$ and using that $k \leq \relent(\bP_W^{S_n} \Vert \bQ_W) + 1$ and $l \leq \bE^{S_n} \big[ [M]_n(W) + 2 \langle M \rangle_n(W) + 1\big]$ yields that, given the event $\cE_{n,k,l}$, with probability smaller or equal than $\bP[\cB_{n,\lambda_{k,l}} | \cE_{n,k,l}]$, there exists some posterior $\bP_W^{S_n}$ such that
    \begin{equation*}
        | \bE^{S_n} M_n(W) | > \frac{2}{\sqrt{6}} \cdot \sqrt{ \bE^{S_n} \big[ [M]_n(W) + \langle M \rangle_n(W) + 1 \big]\Big( \relent(\bP_W^{S_n} \Vert \bQ_W) + \log \frac{2e(n+1)^2}{\beta} \Big)}.
    \end{equation*}
    Now, define $\cB'_n$ as the event stating that there exists some posterior $\bP_W^{S_n}$ such that
    \begin{align*}
        | \bE^{S_n} M_n(W) | > &\bI_\cE \cdot \frac{2}{\sqrt{6}} \cdot \sqrt{ \bE^{S_n} \big[ [M]_n(W) + 2 \langle M \rangle_n(W) + 1 \big]\Big( \relent(\bP_W^{S_n} \Vert \bQ_W) + \log \frac{2e (n+1)^2}{\beta} \Big)} \\
        &+ \bI_{\cE} \cdot \esssup | \bE^{S_n} M_n(W) |,
    \end{align*}
    where $\bP[\cB'_n | \cE_{n,k,l}] \bP[\cE_{n,k,l}] \leq \bP[\cB_{n,\lambda_{k,l}} | \cE_{n,k,l}] \bP[\cE_{n,k,l}] \leq \bP[\cB_{n,\lambda_{k,l}}] < \beta$ for all $(k,l) \in \cK$, and where $\bP[\cB'_n \cap \cE_n^c] = 0$ by the definition of the essential supremum. Therefore, the probability of $\cB'_n$ is bounded as
    \begin{equation*}
        \bP[\cB'_n] = \sum_{(k,l) \in \cK} \bP[\cB'_n | \cE_{n,k,l}] \bP[\cE_{n,k,l}] + \bP[\cB'_n \cap \cE_n^c] < n(1 + \log n) \beta = n \log (en) \beta.
    \end{equation*}
    Finally, let $\beta_n = n \log(en) \beta$ so that, with probability no larger than $\beta_n$, there exists some posterior $\bP_W^{S_n}$ such that
    \begin{align*}
        | \bE^{S_n} &M_n(W) | > \\
        &\bI_\cE \cdot \frac{2}{\sqrt{6}} \cdot \sqrt{ \bE^{S_n} \big[ [M]_n(W) + 2 \langle M \rangle_n(W) + 1 \big]\Big( \relent(\bP_W^{S_n} \Vert \bQ_W) + \log \frac{2en (n+1)^2 \log(en)}{\beta_n} \Big)}  \\
        &+ \bI_{\cE} \cdot \esssup | \bE^{S_n} M_n(W) |.
    \end{align*}
    Finally, the substitution $\beta_n \leftarrow \beta$ completes the proof.
\end{proof}

This technique can be extended to the corollary bound of \citet{haddouche2023pacbayes} for batch learning with i.i.d. data yielding \Cref{th:parameter_free_anytime_valid_bounded_2nd_moment}, where we write $S_n = S$ to simplify the reading in the main text. Note again that we are using the particularization of \citet{haddouche2023pacbayes} with the constants from \citet{wang2015pac}.

Finally, for the second step, \Cref{th:paramter_free_anytime_valid_martingales,th:parameter_free_anytime_valid_bounded_2nd_moment} can be converted back to anytime-valid bounds using \Cref{th:standard_to_anytime_valid}. The resulting bound is exactly the same substituting $\log \xi'(n)$ for $\log \xi''(n)$, where $\xi''(n) \coloneqq \nicefrac{e \pi^2 (n+1)^2 n^3 \log (en)}{3}$.

In case that one desires to have a bound without a $\log n$ term, one may consider employing the technique outlined for~\Cref{th:pac_bayes_chernoff_analogue_no_cutoff} with~\citet{haddouche2023pacbayes}'s~\Cref{th:haddouche_main_th} instead of~\citet{wang2015pac}'s~\Cref{th:wang_main_th}.

\subsection{PAC-Bayes bounds with a smaller union bound cost}
\label{subapp:pac_bayes_smaller_union_bound_cost}

As discussed in \Cref{sec:pac_bayes_beyond_bounded_losses}, the union bound cost of the PAC-Bayes bounds developed above can be improved at the cost of a multiplicative factor of $e$ to the relative entropy. Below, we present the parallel of \Cref{th:pac_bayes_chernoff_analogue} with this improved union bound cost, but extending \Cref{cor:pac_bayes_chernoff_analogue_particular} and \Cref{th:parameter_free_anytime_valid_bounded_2nd_moment}, \ref{th:paramter_free_anytime_valid_martingales}, and \ref{th:pac_alt_rs} follows analogously almost verbatim.

\begin{theorem}
    \label{th:pac_bayes_chernoff_analogue_loglog}
    Consider a loss function $\ell$ with a bounded CGF in the sense of \Cref{ass:bounded_cgf}. Let $\bQ_W$ be any prior independent of $S$ and define the event $\cE = \{ \relent(\bP_W^S \lVert \bQ_W) \leq n \}$. Then, for every $\beta \in (0,1)$ with probability no smaller than $1-\beta$
    \begin{align*}
        \bE^S \poprisk{W} \leq &\bI_{\cE} \cdot \Bigg[ \bE^S \emprisk{W}{S} + \psi_*^{-1} \bigg( \frac{e \max \{\relent(\bP_{W}^S \lVert \bQ_W),1\} + \log \frac{2+\log n}{\beta}} {n} \bigg) \Bigg] \\
        &+ \bI_{\cE^c} \cdot \esssup \bE^S \poprisk{W}
    \end{align*}
    holds \emph{simultaneously} for every posterior $\bP_W^S$.
\end{theorem}

\begin{proof}
    Let $\cB_\lambda$ be the complement of the event in~\eqref{eq:pac_bayes_cgf_with_lambda} such that $\bP[\cB_\lambda] < \beta$ and consider the sub-events $\cE_0 \coloneqq \{\relent(\bP_{W}^S \lVert \bQ_W) \in [0,1] \}$, $\cE_1 \coloneqq \{ \relent(\bP_{W}^S \lVert \bQ_W) \in (1,e] \}$, and $\cE_k \coloneqq \{ \relent(\bP_W^S \lVert \bQ_W) \in (e^{k-1}, e^k] \}$ for all $k = 2, \ldots, n$, which form a covering of the event $\cE \coloneqq \{ \relent(\bP_W^S \lVert \bQ_W) \leq n \}$. Furthermore, define $\cK \coloneqq \{ k \in \bN \cup \{0\}: 0 \leq k \leq n \textnormal{ and } \bP[\cE_k] > 0 \}.$  For all $k \in \cK \setminus \{ 0 \}$, given the event $\cE_k$, with probability no more than $\bP[\cB_\lambda | E_k]$, there exists some posterior $\bP_W^S$ such that
    \begin{equation}
        \label{eq:pac_bayes_cgf_with_lambda_and_k_loglog}
        \bE^S \poprisk{W} > \bE^S \emprisk{W}{S} + \frac{1}{\lambda} \bigg[ \frac{e^k + \log \frac{1}{\beta}}{n} + \psi(\lambda) \bigg],
    \end{equation}
    for all $\lambda \in (0, b)$. The right hand side of \eqref{eq:pac_bayes_cgf_with_lambda_and_k_loglog} can be minimized with respect to $\lambda$ \emph{independently of the training set $S$}. Let $\cB_{\lambda_k}$ be the event resulting from this minimization and note that $\bP[\cB_{\lambda_k}] < \beta$. According to  \citep[Lemma 2.4]{boucheron2013concentration}, this ensures that, with probability no more than $\bP[\cB_{\lambda_k} | \cE_k]$, there exists some posterior $\bP_W^S$ such that
    \begin{equation*}
        \bE^S \poprisk{W} > \bE^S \emprisk{W}{S} +\psi_{*}^{-1} \bigg( \frac{e^k + \log \frac{1}{\beta}}{n} \bigg),
    \end{equation*}
    where $\psi_*$ is the convex conjugate of $\psi$ and where $\psi_*^{-1}$ is a non-decreasing concave function. Given $\cE_k$, since $e^k < e\relent(\bP_W^S \Vert \bQ_W)$, with probability no larger than $\bP[\cB_{\lambda_k} | \cE_k]$, there exists some posterior $\bP_W^S$ such that
    \begin{equation*}
        \bE^S \poprisk{W} > \bE^S \emprisk{W}{S} +\psi_{*}^{-1} \bigg( \frac{e\relent(\bP_W^S \Vert \bQ_W) + \log \frac{1}{\beta}}{n} \bigg).
    \end{equation*}
    Now, define $\cB'$ as the event stating that there exists some posterior $\bP_W^S$ such that
    \begin{equation*}
        \bE^S \poprisk{W} > \bI_{\cE} \cdot \Bigg[ \bE^S \emprisk{W}{S} + \psi_{*}^{-1} \bigg( \frac{e\max \{ \relent(\bP_{W}^{S} \Vert \bQ_W), 1\} + \log \frac{e}{\beta}}{n} \bigg) \Bigg] + \bI_{\cE^c} \cdot \esssup \bE^S \poprisk{W}
    \end{equation*}
    where $\bP[\cB' | \cE_k] \bP[\cE_k] \leq \bP[\cB_{\lambda_k} | \cE_k] \bP[\cE_k] \leq \bP[\cB_{\lambda_k}] \leq \beta$ for all $k \in \cK$, and where $\bP[\cB' \cap \cE^c] = 0$ by the definition of the essential supremum. Note that, if $\{ 0 \} \in \cK$, the case for $k=0$ is handled by the addition of the maximum $\max \{  \relent(\bP_{W}^{S} \Vert \bQ_W), 1\}$ to the equation defining the event $\cB'$. Therefore, the probability of $\cB'$ is bounded as
    \begin{equation*}
        \bP[\cB'] = \sum_{k \in \cK} \bP[\cB' | \cE_k] \bP[\cE_k] + \bP[\cB' \cap \cE^c] < (2 + \log n) \beta.
    \end{equation*}
    Finally, the substitution $\beta \leftarrow \beta/ (2 + \log n)$ completes the proof.
\end{proof}

\subsection{PAC-Bayes bounds without an uninteresting event}
\label{subapp:proof_no_cutoff}

As discussed in \Cref{sec:pac_bayes_beyond_bounded_losses}, the presented approach to find parameter-free PAC-Bayes bounds can be extended to the case where no event is considered uninteresting. Below, we present the parallel of~\Cref{th:pac_bayes_chernoff_analogue} with this consideration. However, extending \Cref{cor:pac_bayes_chernoff_analogue_particular} and \Cref{th:parameter_free_anytime_valid_bounded_2nd_moment,th:paramter_free_anytime_valid_martingales,th:pac_alt_rs} and \ref{th:pac_bayes_chernoff_analogue_loglog} follows analogously almost verbatim. The main important considerations are that to extend \Cref{th:parameter_free_anytime_valid_bounded_2nd_moment,th:paramter_free_anytime_valid_martingales} one needs to consider a double sum $\sum_{k=1}^\infty \sum_{l=1}^\infty \beta_{k,l}$ and that to extend the results with the $\log \log$ cost from \Cref{subapp:pac_bayes_smaller_union_bound_cost} one needs to separate the space with a geometric grid. \\

\noindent{\bf Proof of \Cref{th:pac_bayes_chernoff_analogue_no_cutoff}\ }
Consider the sub-events $\cE_1 \coloneqq \{ \relent(\bP_W^S \Vert \bQ) \leq 1 \}$ and $\cE_k \coloneqq \{ \lceil \relent(\bP_W^S \Vert \bQ) \rceil = k \}$ for all $k \geq 2 \in \bN$, which form a covering of the events' space. Furthermore, define $\cK \coloneqq \{k \in \bN : 1 \leq k \textnormal{ and } \bP[\cE_k] > 0 \}$. For all $k \in \cK$, consider the event $\cB_{\lambda_k}$ to be the complement of the event in~\eqref{eq:pac_bayes_cgf_with_lambda} for a given parameter $\beta_k$ such that $\bP[\cB_{\lambda_k}] < \beta_k$. Then, given the event $\cE_k$, with probability no more than $\bP[\cB_{\lambda_k} | \cE_k]$, there exists some posterior $\bP_{W}^S$ such that
\begin{equation}
    \label{eq:pac_bayes_cgf_with_lambda_and_k_no_cutoff}
    \bE^S \poprisk{W} > \bE^S \emprisk{W}{S} + \frac{1}{\lambda_k} \bigg[ \frac{k + \log \frac{1}{\beta_k}}{n} + \psi(\lambda_k) \bigg],
\end{equation}
for all $\lambda_k \in (0, b)$. The right hand side of \eqref{eq:pac_bayes_cgf_with_lambda_and_k_no_cutoff} can be minimized with respect to $\lambda$ \emph{independently of the training set $S$}. Let $\cB_{\lambda_k}$ be the event resulting from this minimization and note that $\bP[\cB_{\lambda_k}] \leq \beta_k$. According to  \citep[Lemma 2.4]{boucheron2013concentration}, this ensures that with probability no more than $\bP[\cB_{\lambda_k} | \cE_k]$, there exists some posterior $\bP_W^S$ such that
\begin{equation*}
    \bE^S \poprisk{W} > \bE^S \emprisk{W}{S} +\psi_{*}^{-1} \bigg( \frac{k + \log \frac{1}{\beta_k}}{n} \bigg),
\end{equation*}
where $\psi_*$ is the convex conjugate of $\psi$ and where $\psi_*^{-1}$ is a non-decreasing concave function. Now, let $\beta_k = \frac{\beta}{k^2}$. Given $\cE_k$, since $k < \relent(\bP_W^S \Vert \bQ_W) + 1$, with probability no larger than $\bP[\cB_{\lambda_k} | \cE_k]$, there exists some posterior $\bP_W^S$ such that
\begin{equation}
    \label{eq:event_to_bound_no_cutoff}
    \bE^S \poprisk{W} > \bE^S \emprisk{W}{S} +\psi_{*}^{-1} \Bigg( \frac{\relent(\bP_W^S \Vert \bQ_W) + 1 + \log \frac{\big( \relent(\bP_W^S \Vert \bQ_W) + 1  \big)^2}{\beta}}{n} \Bigg).
\end{equation}
Now, define $\cB'$ as the event described in~\eqref{eq:event_to_bound_no_cutoff}, where  $\bP[\cB' | \cE_k] \bP[\cE_k] \leq \bP[\cB_{\lambda_k} | \cE_k] \bP[\cE_k] \leq \bP[\cB_{\lambda_k}] \leq \beta_k = \frac{\beta}{k^2}$ for all $k \in \cK$. Therefore, the probability of $\cB'$ is bounded as
\begin{equation*}
    \bP[\cB'] = \sum_{k \in \cK} \bP[\cB' | \cE_k] \bP[\cE_k] < \sum_{k=1}^\infty \frac{\beta}{k^2} = \frac{\pi^2}{6} \cdot \beta.
\end{equation*}
Finally, the substitution $\beta \leftarrow \nicefrac{6\beta}{\pi^2}$ completes the proof.
\hfill\BlackBox\\[2mm]

\subsection{Example: Recovering a PAC-Bayes bound on the randomized subsample setting}
\label{subapp:recovering_randomized_subsample_setting_pac_bayes_bound}

In this section of the Appendix, as a further application of the technique devised to prove~\Cref{th:pac_bayes_chernoff_analogue} to obtain bounds ``in probability'', we recover~\citet{hellstrom2020generalization}'s PAC-Bayes bound in the \emph{randomized subsample setting}. This setting was introduced by \citet{steinke2020reasoning} motivated by the fact that the dependency measure $\relent(\bP_{W|S} \lVert \bQ_W)$ can be large, or even infinite, in situtations where algorithms generalize~\citep{bassily2018learners, livni2020limitation}.

In the randomized subsample setting, it is considered that the training dataset $S$ is obtained through the following mechanism. First, a super sample
\begin{equation*}
    \tilde{S} \coloneqq 
    \begin{pmatrix}
        \tilde{Z}_{1,1} & \tilde{Z}_{2,1} & \cdots & \tilde{Z}_{n,1} \\
        \tilde{Z}_{1,2} & \tilde{Z}_{2,2} & \cdots & \tilde{Z}_{n,2}
    \end{pmatrix}^\intercal
\end{equation*}
of $2n$ i.i.d. instances is obtained by sampling from the distribution $\bP_Z$. Unused samples, or \emph{ghost samples} $S_\textnormal{ghost} = \tilde{S} \setminus S$, are virtual and only exist for the purpose of the analysis. Then, an independent sequence of indices $U \coloneqq (U_1, \ldots, U_n)$ distributed as $\bP_{U_i}[1] = \bP_{U_i}[2] = \nicefrac{1}{2}$ is generated. Finally, the training set is sub-sampled from the superset so that $Z_i = \tilde{Z}_{i,U_i}$. Under this setting, \citet{steinke2020reasoning} and \citet{grunwald2021pac} derive PAC-Bayes and MAC-Bayes bounds, where the dependence measure is now the relative entropy $\relent(\bP_{W}^S \lVert \bQ_{W}^{\tilde{S}})$ %
of the posterior $\bP_{W}^{S}$ with respect to a \emph{data-dependent} prior $\bQ_{W}^{\tilde{S}}$ that has access to the supersample $\tilde{S}$ but not to the indices $U$.

Bounds depending on this ``conditional'' measure are promising. \citet{steinke2020reasoning} and \citet{grunwald2021pac} showed that for finite VC dimension and compression schemes both MAC- and PAC-Bayes bounds can be obtained; \citet{haghifam2021towards} showed that a highly related measure provides a sharp characterization of the population risk in the realizable setting for 0--1 losses; and \citet{hellstrom2022new} made a similar remark for both PAC- and MAC-Bayes bounds for classes with finite Natarajan dimension~\citep{natarajan1989learning}. However, recently~\citet{haghifam2022limitations} showed that there are still simple algorithms that generalize but where this conditional measure is high.

In this setting, the canonical PAC-Bayes bound is given by \citet{hellstrom2021corrections}.

\begin{theorem}[{\bf \citet[Theorem 3 and Corollary 6]{hellstrom2020generalization,hellstrom2021corrections}}] 
    \label{th:pac_bayes_randomized_subsample}
    Consider a bounded loss function $\ell: \cW \times \cZ \to [a, b]$. Let $\bQ_{W}^{\tilde{S}}$ be any Markov kernel on $\cW$ with access to the supersample $\tilde{S}$ but not to the indices $U$. Then, for all $\beta \in (0,1)$, with probability no less than $1-\beta$
    \begin{align*}
        \bE^S \poprisk{W} \leq \bE^S \emprisk{W}{S} + \sqrt{\frac{2(b-a)^2 \big(\relent(\bP_{W}^{S} \lVert \bQ_{W}^{\tilde{S}}) + \log \frac{2\sqrt{n}}{\beta}  \big) }{n-1} } + \sqrt{ \frac{(b-a)^2 \log \frac{4}{\beta}  }{2 n} }
    \end{align*}
    holds \emph{simultaneously} for every posterior $\bP_W^S$,
    where the conditioning is written on $S$ and not on $\tilde{S}$ and $U$ since $\bP_{W}^{\tilde{S},U} = \bP_{W}^{S}$ a.s..
\end{theorem}

To obtain the PAC-Bayes bound from~\Cref{th:pac_bayes_randomized_subsample}, \citet{hellstrom2021corrections} use a specific property of subgaussian random variables~\citep[Theorem 2.6]{wainwright2019high}. In the randomized subsample setting, where the loss is assumed to be bounded, this results in a general bound as bounded random variables are subgaussian. Still, to illustrate the technique from \Cref{th:pac_bayes_chernoff_analogue}, we show how an equivalent result is obtained based on~\citet{hellstrom2020generalization}'s original procedure. 

Using \citet[Theorem 4]{hellstrom2020generalization}'s exponential inequality, which independently re-discovered \citep[Lemma 2.1]{zhang2006information}, and following the steps of \citep[Corolary 2]{hellstrom2020generalization} yields the next result.

\begin{lemma} 
    \label[lemma]{th:pac_bayes_randomized_subsample_with_lambda}
    Consider a bounded loss function $\ell: \cW \times \cZ \to [a, b]$. Let $\bQ_{W}^{\tilde{S}}$ be any Markov kernel on $\cW$ with access to the supersample $\tilde{S}$ but not to the indices $U$. For all $\lambda \in \bR_+$ and every $\beta \in (0,1)$, with probability no less than $1-\beta$
    \begin{align}
        \label{eq:pac_bayes_randomized_subsample_with_lambda}
        \bE^{\tilde{S},U} \emprisk{W}{S_\textnormal{ghost}} \leq \bE^S \emprisk{W}{S} + \frac{1}{\lambda} \bigg[ \relent( \bP_{W}^{S} \Vert \bQ_{W}^{\tilde{S}} ) + \log \frac{1}{\beta}  \bigg] + \frac{\lambda c^2}{2n}
    \end{align}
    holds \emph{simultaneously} for every posterior $\bP_W^S$.
\end{lemma}

Comparing~\eqref{eq:pac_bayes_cgf_with_lambda} and~\eqref{eq:pac_bayes_randomized_subsample_with_lambda}, it is apparent that the technique from~\Cref{th:pac_bayes_chernoff_analogue} can be readily applied. However, as the function $\psi$ is explicit here, to highlight the difference between this approach and the one quantizing  the parameter space~\citep[Section 2.1.4]{alquier2021user}, we show how to proceed below. This also showcases when one would want to choose a more stringent event as anticipated in \Cref{subsubsec:event_cutoff}.

\begin{theorem} 
    \label{th:pac_alt_rs}
    Consider a bounded loss function $\ell: \cW \times \cZ \to [a, b]$. Let $\bQ_{W}^{\tilde{S}}$ be any Markov kernel on $\cW$ with access to the supersample $\tilde{S}$ but not to the indices $U$. Then, for every $\beta \in (0,1)$, with probability no smaller than $1-\beta$
    \begin{align*}
        \bE^S \poprisk{W} \leq \bE^S \emprisk{W}{S} + \sqrt{\frac{2(b-a)^2 \big(\relent(\bP_W^S \lVert \bQ_{W|\tilde{S}}) + \log \frac{en}{\beta} \big) }{n} } + \sqrt{ \frac{(b-a)^2 \log( \frac{4}{\beta} ) }{2n} }
    \end{align*}
    holds \emph{simultaneously} for every posterior $\bP_W^S$.
\end{theorem}

\begin{proof}
Similarly to the proof of~\Cref{th:pac_bayes_chernoff_analogue}, let $\cB_\lambda$ be the complement of the event in~\eqref{eq:pac_bayes_randomized_subsample_with_lambda}, and define the event $\cE \coloneqq \{ \relent(\bP_W^S \lVert \bQ_W) \leq k_\textnormal{max} \}$,  and the sub-events $\cE_1 \coloneqq \{ \relent(\bP_W^S \Vert \bQ_W^{\tilde{S}}) \leq 1 \}$ and  $\cE_k \coloneqq \{ \lceil \relent(\bP_W^S \Vert \bQ_W^{\tilde{S}}) \rceil = k \}$ for $k = 2, \ldots, k_\textnormal{max}$. Define also $\cK \coloneqq \{ k \in \bN: 1 \leq k \leq k_\textnormal{max} \textnormal{ and } \bP[\cE_k] > 0 \}$. Then, for all $k \in \cK$, given the event $\cE_k$, with probability at most $\bP[\cB_\lambda | \cE_k]$, there exists some posterior $\bP_W^S$ such that
\begin{equation*}
    \label{eq:pac_bayes_randomized_subsample_with_lambda_and_k}
    \bE^{\tilde{S},U} \emprisk{W}{S_\textnormal{ghost}}  > \bE^S \emprisk{W}{S} + \frac{1}{\lambda} \bigg[ k + \log \frac{1}{\beta}  \bigg] + \frac{\lambda (b-a)^2}{2n}    
\end{equation*}
for all $\lambda > 0$. The above equation is optimized for
\begin{equation*}
    \label{eq:optimal_lambda}
    \lambda = \lambda_k \coloneqq \sqrt{\frac{2n}{(b-a)^2} \Big(k + \log \frac{1}{\beta} \Big)}
\end{equation*}
and using that $k \leq \relent( \bP_W^S \lVert \bQ_{W|\tilde{S}} ) + 1$ yields that, given the event $\cE_k$, with probability smaller or equal than $\bP[\cB_{\lambda_k} | \cE_k]$, there exists some posterior $\bP_W^S$ such that
\begin{equation}
    \label{eq:solution_rs_event_conditioned}
    \bE^{\tilde{S},U} \emprisk{W}{S_\textnormal{ghost}} > \bE^S \emprisk{W}{S} + \sqrt{\frac{ 2(b-a)^2 \big(\relent( \bP_{W}^S \lVert \bQ_{W}^{\tilde{S}} ) + \log\frac{1}{\beta} + 1  \big) }{n}}.
\end{equation}
Let $\cB'$ be the event in~\eqref{eq:solution_rs_event_conditioned} and note that $\bP[B' \cap E^c] = 0$ as long as $k_\textnormal{max} \geq \frac{n}{2} - 1$ by the boundedness of the loss. Solving as in \Cref{th:pac_bayes_chernoff_analogue} and using \citep[Theorem 3]{hellstrom2020generalization} completes the proof. 
\end{proof}

The parameter $\lambda$ optimization strategies from \citet{langford2001not} and \citet{catoni2003pac} are based on a quantization of the \emph{parameter space}. They choose a countable set $\cA \subset \bR_+$ and use other techniques (e.g., rounding) to ensure that the optimization can be done on a larger set $\cA'$ with a union bound price of $\log |\cA|$~\citep[Section 2.1.4]{alquier2021user}. However, it is not always certain that the optimal parameter lies on the extended set $\cA'$, and thus a parameter-free closed form expression of the bound is unattainable.
Here, we instead partition \emph{the set of events described by the random variables}, create a further upper bound that depends only on deterministic terms for each event, and select a deterministic parameter for that event. Then, we pay a union bound price for each event considered. Therefore, as mentioned previously, our technique can be seen as an optimization of the set $\cA^\star \coloneqq \{ \lambda_0, \ldots, \lambda_{k_\textnormal{max}} \}$ of parameters for which we know that an \emph{almost optimal} bound can be reached, where the distance to optimality is given by how loose is the upper bound depending only on deterministic terms~\eqref{eq:pac_bayes_randomized_subsample_with_lambda_and_k}.

\vskip 0.2in
\bibliography{references}

\end{document}